\documentclass{article} % For LaTeX2e
\usepackage{iclr2020_conference}
\usepackage{times}

\iclrfinalcopy

\usepackage{amsmath,amsfonts,bm}

\def\figref#1{Figure~\ref{#1}}
\def\secref#1{Section~\ref{#1}}
\def\eqref#1{Equation~\ref{#1}}
\def\Algref#1{Algorithm~\ref{#1}}

\def\1{\bm{1}}

\def\eps{{\epsilon}}

\def\rb{{\textnormal{b}}}

\def\vb{{\bm{b}}}

\DeclareMathAlphabet{\mathsfit}{\encodingdefault}{\sfdefault}{m}{sl}
\SetMathAlphabet{\mathsfit}{bold}{\encodingdefault}{\sfdefault}{bx}{n}

\def\gG{{\mathcal{G}}}

\def\gN{{\mathcal{N}}}

\def\gZ{{\mathcal{Z}}}

\def\sP{{\mathbb{P}}}

\newcommand{\E}{\mathbb{E}}

\DeclareMathOperator*{\argmin}{arg\,min}

\let\ab\allowbreak

\usepackage{Definitions}

\usepackage{hyperref}
\usepackage{url}
\usepackage{graphicx}
\usepackage{xspace}
\usepackage{multirow}
\usepackage{enumitem}
\usepackage{todonotes}
\usepackage{wrapfig}
\usepackage{booktabs}
\usepackage{adjustbox}
\usepackage{algorithm}
\usepackage{algorithmic}

\usepackage{color}
\usepackage{xcolor}
\newcommand{\newadd}[1]{{ #1}} % \color{cyan}\\
\newcommand{\Bo}[1]{{\color{red} [Bo: #1]}}

\newcommand{\RN}[1]{%
	\textup{\lowercase\expandafter{\it \romannumeral#1}}%
}

\newcommand{\estname}{{GenDICE}\xspace}
\newcommand{\estabb}{{GenDICE}\xspace}
\newcommand{\Dset}{\mathcal{D}}
\newcommand{\pval}{\mathcal{R}}
\newcommand{\spn}{\operatorname{span}}

\renewcommand{\cite}{\citep}

\def\ptau{p\cdot\tau}
\def\Ttau{\rbr{\Tpgam \circ\tau}}
\def\Tpgam{\Tcal_{\gamma, \mu_0}^p}
\def\tauhs{\hat\tau_{\Hcal}^*}
\def\uhs{\uhat^*}
\def\fhs{\fhat_\Fcal^*}

\title{\estabb: Generalized Offline Estimation of Stationary Values}

\author{Ruiyi Zhang$^{1}$\thanks{Equal contribution.}\hspace{2mm}\thanks{Work done while interning at Google.}\hspace{1mm},
  Bo Dai$^{2*}$, Lihong Li$^2$, Dale Schuurmans$^2$ \\
  $^1$Duke University, $^2$Google Research, Brain Team\\	
}

\begin{document}

\maketitle
%!TEX root = dual_ratio.tex

\begin{abstract}

An important problem that arises in reinforcement learning and Monte Carlo methods is estimating quantities defined by the stationary distribution of a Markov chain. In many real-world applications, access to the underlying transition operator is limited to a fixed set of data that has already been collected, without additional interaction with the environment being available. We show that consistent estimation remains possible in this challenging scenario, and that effective estimation can still be achieved in important applications. Our approach is based on estimating a ratio that corrects for the discrepancy between the stationary and empirical distributions, derived from fundamental properties of the stationary distribution, and exploiting constraint reformulations based on variational divergence minimization. The resulting algorithm, \estname, is straightforward and effective. We prove its consistency under general conditions, provide an error analysis, and demonstrate strong empirical performance on benchmark problems, including off-line PageRank and off-policy policy evaluation.

\end{abstract}

%!TEX root = dual_ratio.tex

%%%%%%%%%%%%%%%%%%%%%%%%%%%%%%%%%%%%%%%%%%%%%%%%%%%%%%%%%%%%%%%%%%%%%%%%%%%%
\section{Introduction}\label{sec:intro}
%%%%%%%%%%%%%%%%%%%%%%%%%%%%%%%%%%%%%%%%%%%%%%%%%%%%%%%%%%%%%%%%%%%%%%%%%%%%

Estimation of quantities defined by the stationary
distribution of a Markov chain
lies at the heart of many scientific and engineering problems.
Famously, the steady-state distribution of a random walk on the World Wide Web
provides the foundation of the %celebrated 
PageRank
algorithm~\citep{langville04deeper}.
In many areas of machine learning,
Markov chain Monte Carlo (MCMC) methods are used to 
conduct approximate Bayesian inference
by considering Markov chains whose equilibrium distribution is a desired
posterior~\citep{andrieu02introduction}. 
An example from engineering is queueing theory,
where the queue lengths and waiting time under
the limiting distribution have been extensively
studied~\citep{gross18fundamentals}.
As we will also see below, 
stationary distribution quantities are
of fundamental importance in reinforcement learning 
(RL)~\citep[e.g.,][]{tsitsiklis97analysis}.
%  Example in SDE?

Classical algorithms for estimating stationary distribution quantities
rely on the ability to sample next states from the current state 
\emph{by directly interacting with the environment}
(as in on-line RL or MCMC),
or even require the transition probability distribution to be given explicitly
(as in PageRank).
%MCMC methods also typically require a long trajectory of the
%Markov chain to be gathered to allow a state sampled after the burn-in period
%to be approximately sampled from the stationary distribution.  
Unfortunately, these classical approaches are inapplicable
when direct access to the environment is not available,
which is often the case in practice.
There are many practical scenarios where
a collection of sampled trajectories is available,
having been collected off-line by an external mechanism
that chose states and recorded the subsequent next states.
Given such data,
we still wish to estimate a stationary quantity.
One important example is off-policy policy evaluation in RL,
where we wish to estimate the value of a policy different from that
used to collect experience.
Another example is off-line PageRank (OPR), 
where we seek to estimate the relative importance of webpages given a
sample of the web graph.

Motivated by the importance of these off-line scenarios,
and by the inapplicability of classical methods,
we study the problem of \emph{off-line estimation of stationary values}
via a \emph{stationary distribution corrector}.
Instead of having access to the transition probabilities or a next-state
sampler, we assume only access to a \emph{fixed} sample of state transitions,
where states have been sampled from an unknown distribution
and next-states are sampled according to the Markov chain's transition operator.
%This off-line setting is distinct from that considered by most MCMC or
%on-line RL methods, where it is assumed that new observations can be
%continually sampled by demand from the environment.
The off-line setting is indeed more challenging than its 
more traditional on-line counterpart,
given that one must infer an asymptotic quantity from finite data.
Nevertheless, we develop techniques that still allow consistent estimation
under general conditions, and provide effective estimates in practice.
The main contributions of this work are:
\vspace{-2mm}
\begin{itemize}[leftmargin=*]
\item
We formalize the problem of off-line estimation of stationary quantities,
which captures a wide range of practical applications.
\item
We propose a novel stationary distribution estimator, \estname, for this task.
The resulting algorithm is
% We propose a novel algorithm for this task, \estname, by stationary distribution correction estimation. 
based on a new dual embedding formulation for divergence minimization,
with a carefully designed mechanism that explicitly eliminates degenerate solutions.
\item
We theoretically establish consistency and other statistical properties
of~\estname, and empirically demonstrate that it achieves significant
improvements on several behavior-agnostic off-policy evaluation benchmarks and an off-line version of PageRank.
\end{itemize}
\vspace{-2mm}
The methods we develop in this paper fundamentally extend recent work in
off-policy policy evaluation \citep{liu2018breaking,nachum2019dualdice}
by introducing a new formulation that leads to a more general,
and as we will show, more effective estimation method.

%!TEX root = dual_ratio.tex

\vspace{-3mm}
\section{Background}
\vspace{-2mm}
\label{eq:prelim}
\label{eq:background}

% \Bo{We need to keep notation consistent. I will use $\mu\rbr{x'} = \rbr{\Tcal\circ\mu}\rbr{x}dx $ in following section for unification. For $\gamma = 1$ and $\gamma \in (0, 1)$ we need to specify the definition of $\Tcal\circ \mu$. With such notations, we take the OPE and OPR as major application examples. }

We first introduce off-line PageRank (OPR)
and off-policy policy evaluation (OPE)
as two motivating domains,
where the goal is to estimate stationary quantities
given only off-line access to a set of sampled transitions from an environment.

\vspace{-3mm}
\paragraph{Off-line PageRank (OPR)}
The celebrated PageRank algorithm \citep{PagBriMotWin99}
defines the ranking of a web page in terms of its asymptotic visitation
probability under a random walk on the (augmented) directed graph
specified by the hyperlinks.
If we denote the World Wide Web by a directed graph $G = \rbr{V, E}$
with vertices (web pages) $v\in V$ and edges (hyperlinks) $\rbr{v, u}\in E$,
PageRank considers the random walk defined by the Markov transition operator
$v\rightarrow u$:
\begin{equation}
\textstyle
\Pb\rbr{u|v} =
\frac{(1-\eta)}{\abr{v}}\one_{\rbr{v, u}\in E} + \frac{\eta}{\abr{V}}
\,,
\label{eq:pagerank}
\end{equation}
where $\abr{v}$ denotes the out-degree of vertex $v$
and $\eta\in[0,1)$ is a probability of ``teleporting" to any page uniformly.
Define
$d_t\rbr{v} \defeq
\PP\rbr{s_t=v| s_0\sim \mu_0, \forall i< t, s_{i+1}\sim\Pb(\cdot|s_i)}$,
where $\mu_0$ is the initial distribution over vertices,
then the original PageRank algorithm explicitly iterates for the limit
\begin{equation}
d\rbr{v}\defeq
  \begin{cases}
    \lim_{t\rightarrow\infty} d_t\rbr{v}
      & \quad \text{if } \gamma =1
    \\[0.5ex]
    (1-\gamma)\sum_{t=0}^{\infty}\gamma^t d_t\rbr{v}
      & \quad \text{if } \gamma \in (0, 1)\,.
  \end{cases}
\end{equation}
% \vspace{-1mm}
The classical version of this problem
is solved by tabular methods that simulate \eqref{eq:pagerank}.
However, we are interested in a more scalable off-line version of the problem
where the transition model is not explicitly given.
Instead, consider estimating the rank of a particular web page $v'$
from a large web graph, given only a sample $\Dcal = \cbr{\rbr{v, u}_i}_{i=1}^N$
from a random walk on $G$ as specified above.
We would still like to estimate $d(v')$ based on this data.
First, note that if one knew the distribution $p$ by which any vertex $v$
appeared in $\Dcal$,
the target quantity could be re-expressed by a simple importance ratio
$d\rbr{v'} =\EE_{v\sim p}\sbr{\frac{d\rbr{v}}{p\rbr{v}}\one_{v=v'}}$.
Therefore, if one had the correction ratio function 
$\tau\rbr{v}=\frac{d\rbr{v}}{p\rbr{v}}$, 
an estimate of $d\rbr{v'}$ can easily be recovered
via $d\rbr{v'}\approx\hat p\rbr{v'}\tau\rbr{v'}$,
where $\hat p\rbr{v'}$ is the empirical probability of $v'$ 
estimated from $\Dcal$.
The main attack on the problem we investigate is to recover a good
estimate of the ratio function $\tau$.

\vspace{-3mm}
\paragraph{Policy Evaluation}

An important generalization of this stationary value estimation problem
arises in RL in the form of policy evaluation.
Consider a Markov Decision Process~(MDP) 
$\mathcal{M} = \langle S, A, \Pb, R, \gamma, \mu_0 \rangle$~\citep{Puterman14},
where $S$ is a state space, $A$ is an action space,
$\Pb\rbr{s'|s, a}$ denotes the transition dynamics,
$R$ is a reward function, $\gamma\in(0, 1]$ is a discounted factor,
and $\mu_0$ is the initial state distribution.
Given a policy, which chooses actions in any state $s$ according to the
probability distribution $\pi(\cdot|s)$,
a trajectory $\beta=(s_0,a_0,r_0,s_1,a_1,r_1,\ldots)$
is generated by first sampling the initial state $s_0 \sim \mu_0$,
and then for $t \ge 0$, $a_t \sim \pi(\cdot|s_t)$, $r_t \sim R(s_t,a_t)$,
and $s_{t+1} \sim \Pb(\cdot|s_t,a_t)$.
The value of a policy $\pi$ is the expected per-step reward defined as:
\begin{equation}\label{eq:rewardfunc}
\textstyle
\text{\small{Average:}}~~\pval(\pi)\defeq\lim_{T\rightarrow\infty} \frac{1}{T+1}\EE\sbr{\sum_{t=0}^{T} r_t}\,,
~~\text{\small{Discounted:}}~~\pval_\gamma(\pi)\defeq (1-\gamma)\EE\sbr{\sum_{t=0}^{\infty}\gamma^t r_t}\,.
\end{equation}
In the above, the expectation is taken with respect to the randomness in 
the state-action pair $\Pb\rbr{s'|s, a}\pi\rbr{a'|s'}$
and the reward $R\rbr{s_t, a_t}$.
Without loss of generality, we assume the limit exists for the average case,
and hence $\pval(\pi)$ is finite.

\vspace{-3mm}
\paragraph{Behavior-agnostic Off-Policy Evaluation~(OPE)}

An important setting of policy evaluation that often arises in practice
is to estimate $\pval_\gamma\rbr{\pi}$ or $\pval\rbr{\pi}$
given a fixed dataset 
$\Dcal = \cbr{\rbr{s, a, r, s'}_i}_{i=1}^N\sim \Pb\rbr{s'|s, a}p\rbr{s, a}$,
where $p\rbr{s, a}$ is an unknown distribution
induced by multiple unknown behavior policies.
This problem is different from the classical form of OPE,
where it is assumed that a known behavior policy $\pi_b$
is used to collect transitions.
In the behavior-agnostic scenario, however, 
typical importance sampling~(IS)~estimators~\citep[e.g.,][]{PreSutSin00}
do not apply.
Even if one can assume $\Dcal$ consists of trajectories
where the behavior policy can be estimated from data,
it is known that that straightforward IS estimators suffer
a variance exponential in the trajectory length, known as
the ``curse of horizon''~\citep{jiang2015doubly,liu2018breaking}.

Let $d^\pi_t\rbr{s, a} = \PP\rbr{s_t=s, a_t=a| s_0\sim \mu_0, \forall i< t, a_i\sim\pi\rbr{\cdot|s_i}, s_{i+1}\sim\Pb(\cdot|s_i, a_i)}$. 
The stationary distribution can then be defined as
\begin{equation}
\label{eq:mdp_stationary}
\mu_\gamma^\pi\rbr{s, a}\defeq
  \begin{cases}
    \lim_{T\rightarrow\infty} \frac{1}{T+1}\sum_{t=0}^T d^\pi_t\rbr{s, a}
    = \lim_{t\rightarrow\infty} d^\pi_t\rbr{s, a}
      & \quad \text{if } \gamma =1
    \\[0.5ex]
    (1-\gamma)\sum_{t=0}^{\infty}\gamma^t d^\pi_t\rbr{s, a}
      & \quad \text{if } \gamma \in (0, 1)\,.
  \end{cases}
\end{equation}
With this definition, $\pval(\pi)$ and $\pval_\gamma\rbr{\pi}$ 
can be equivalently re-expressed as
\begin{equation}
\label{eq:dual_reward}
\textstyle
\pval_\gamma(\pi)\defeq \EE_{\mu_\gamma^\pi}\sbr{R\rbr{s, a}}
= \EE_{p}\sbr{\frac{\mu_\gamma^\pi\rbr{s, a}}{p\rbr{s, a}}R\rbr{s, a}}\,.
\end{equation}
Here we see once again that if we had the correction ratio function
$\tau\rbr{s,a}=\frac{\mu_\gamma^\pi\rbr{s, a}}{p\rbr{s, a}}$
a straightforward estimate of $\pval_\gamma(\pi)$ could be recovered via
$\pval_\gamma(\pi)\approx\EE_{\hat p}\sbr{\tau\rbr{s, a}R\rbr{s,a}}$,
where $\hat p\rbr{s,a}$ is an empirical estimate of $p\rbr{s,a}$.
In this way, the behavior-agnostic OPE problem can be
reduced to estimating the correction ratio function $\tau$, as above.

We note that \citet{liu2018breaking} and~\citet{nachum2019dualdice}
also exploit~\eqref{eq:dual_reward}
to reduce OPE to stationary distribution correction,
but these prior works are distinct from the current proposal
in different ways.
First, the inverse propensity score~(IPS)~method of \citet{liu2018breaking}
assumes the transitions are sampled from a \emph{single} behavior policy,
which must be \emph{known} beforehand;
hence that approach is not applicable in behavior-agnostic OPE setting.
Second,
the recent DualDICE~algorithm \citep{nachum2019dualdice}
is also a behavior-agnostic OPE estimator,
but its derivation relies on a \emph{change-of-variable} trick
that is only valid for $\gamma<1$.
This previous formulation becomes unstable when $\gamma\rightarrow 1$,
as shown in~\secref{sec:experiments} and~\appref{appendix:exp}.
The behavior-agnostic OPE estimator we derive below in~\secref{sec:dual_est}
is applicable both when $\gamma = 1$ and $\gamma\in (0, 1)$.
This connection is why we name the new estimator \estabb,
for \emph{GENeralized stationary DIstribution Correction Estimation}.
%\Bo{Should we modify this considering the new title?}

%!TEX root = dual_ratio.tex

%%%%%%%%%%%%%%%%%%%%%%%%%%%%%%%%%%%%%%%%%%%%%%%%%%%%%%%%%%%%%%%%%%%%%%%%%%%%
\vspace{-3mm}
\section{\estname}\label{sec:dual_est}
\vspace{-2mm}
%%%%%%%%%%%%%%%%%%%%%%%%%%%%%%%%%%%%%%%%%%%%%%%%%%%%%%%%%%%%%%%%%%%%%%%%%%%%

As noted, there are important estimation problems in the Markov chain and MDP
settings that can be recast as estimating a stationary distribution
correction ratio.
We first outline the conditions that characterize the 
correction ratio function $\tau$,
upon which we construct the objective for the~\estabb estimator,
and design efficient algorithm for optimization.
We will develop our approach for the more general MDP setting,
with the understanding that all the methods and results
can be easily specialized to the Markov chain setting. % as necessary.

%--------------------------------------------------------------------------
\vspace{-2mm}
\subsection{Estimating Stationary Distribution Correction}\label{sec:unbias_estimator}
\vspace{-1mm}
%--------------------------------------------------------------------------

The stationary distribution $\mu^\pi_\gamma$ defined
in~\eqref{eq:mdp_stationary} can also be characterized via
\begin{equation}
\label{eq:stationary}
% \textstyle
\resizebox{0.85\hsize}{!}{$
\mu\rbr{s', a'} = \underbrace{\rbr{1 - \gamma}\mu_0\rbr{s'}\pi\rbr{a'|s'} + \gamma\int \pi\rbr{a'|s'}\Pb\rbr{s'|s, a}\mu\rbr{s, a}ds\,da}_{\rbr{\Tcal\circ \mu}\rbr{s', a'}},\,\,\forall \rbr{s', a'}\in S\times A.
$}
\end{equation}
At first glance, this equation shares a superficial similarity to the Bellman
equation, but there is a fundamental difference.
The Bellman operator recursively integrates out future $\rbr{s', a'}$ pairs
to characterize a current pair $\rbr{s, a}$ value,
whereas the distribution operator $\Tcal$ 
defined in \eqref{eq:stationary}
operates in the reverse temporal direction.

When $\gamma<1$, \eqref{eq:stationary} always has a fixed-point solution.
For $\gamma =1$, in the discrete case, the fixed-point exists
as long as $\Tcal$ is ergodic;
in the continuous case, the conditions for fixed-point existence
become more complicated~\citep{MeyTwe12} and beyond the scope of this paper.

The development below is based on a divergence $D$ and
the following default assumption.
\vspace{-2mm}
\begin{assumption}[Markov chain regularity]\label{asmp:stat_exist}
For the given target policy $\pi$,
the resulting state-action transition operator $\Tcal$
has a unique stationary distribution $\mu$ that satisfies
$D(\Tcal\circ\mu\|\mu)=0$.
\end{assumption}
\vspace{-1mm}
In the behavior-agnostic setting we consider,
one does not have direct access to $\Pb$ for element-wise
evaluation or sampling,
but instead 
is given
a fixed set of samples from $\Pb\rbr{s'|s, a}p\rbr{s, a}$ 
with respect to some distribution $p\rbr{s,a}$ over $S\times A$.
Define $\Tpgam$ to be a mixture of $\mu_0\pi$ and $\Tcal_p$; \ie, let
\begin{equation}\label{eq:ref_tp}
\Tpgam\rbr{\rbr{s', a'}, \rbr{s, a}} \defeq  \rbr{1 - \gamma}\mu_0\rbr{s'}\pi\rbr{a'|s'} +\gamma \underbrace{\pi\rbr{a'|s'}\Pb\rbr{s'|s, a}p\rbr{s, a}}_{\Tcal_p\rbr{\rbr{s', a'}, \rbr{s, a}}}
.
\end{equation}
Obviously, conditioning on $\rbr{s, a, s'}$
one could easily sample $a'\sim \pi\rbr{a'|s'}$
to form
$\rbr{s, a, s', a'}\sim\Tcal_p\rbr{\rbr{s', a'}, \rbr{s, a}}$;
similarly,
a sample $\rbr{s', a'}\sim \mu_0\rbr{s'}\pi\rbr{a'|s'}$
could be formed from $s'$.
Mixing such samples with probability $\gamma$ and $1-\gamma$ respectively
yields a sample $\rbr{s, a, s', a'}\sim \Tpgam\rbr{\rbr{s', a'}, \rbr{s, a}}$.
Based on these observations,
the stationary condition for the ratio from~\eqref{eq:stationary}
can be re-expressed in terms of $\Tpgam$ as
\vspace{-1mm}
\begin{equation}\label{eq:stationary_ratio}
% \textstyle
\resizebox{0.9\hsize}{!}{$
p\rbr{s', a'}\tau^*\rbr{s', a'} =\underbrace{\rbr{1 - \gamma}\mu_0\rbr{s'}\pi\rbr{a'|s'} + \gamma \int\pi\rbr{a'|s'}\Pb\rbr{s'|s, a}p\rbr{s, a}\tau^*\rbr{s, a}ds\,da}_{\rbr{\Tpgam\circ\tau^*}\rbr{s', a'}},
$}
\end{equation}
where $\tau^*\rbr{s, a}\defeq \frac{\mu\rbr{s, a}}{p\rbr{s, a}}$
is the correction ratio function we seek to estimate.
One natural approach to  estimating $\tau^*$ is to match the
LHS and RHS of~\eqref{eq:stationary_ratio} 
with respect to some divergence $D\rbr{\cdot\|\cdot}$
over the empirical samples.
That is, we consider estimating $\tau^*$ by solving the optimization problem
\begin{equation}
\label{eq:naive_est}
\min_{\tau\ge 0}\,\,D\rbr{\Tpgam\circ\tau\|\ptau}.
\end{equation}
Although this forms the basis of our approach,
there are two severe issues with this naive formulation
that first need to be rectified:
\begin{itemize}[leftmargin=*,topsep=0pt, nosep]

\item[{\bf i)}] {\bf Degenerate solutions: }
When $\gamma =1$, the operator $\Tcal_{\gamma=1,\mu_0}^p$ is invariant
to constant rescaling:
if $\tau^* = \Tcal_{\gamma=1,\mu_0}^p\circ\tau^*$
then $c\tau^* = \Tcal_{\gamma=1,\mu_0}^p\circ \rbr{c\tau^*}$ for any $c\ge 0$.
Therefore, simply minimizing the divergence
$D\rbr{\Tcal_{\gamma=1,\mu_0}^p\circ\tau\|\ptau}$
cannot provide a desirable estimate of $\tau^*$.
In fact, in this case the trivial solution
$\tau^*\rbr{s,a} = 0$
cannot be eliminated.

\item[{\bf ii)}] {\bf Intractable objective: }
The divergence $D\rbr{\Tpgam\circ\tau\|\ptau}$ involves the computation
of $\Tpgam\circ\tau$, which in general involves an intractable integral.
Thus, evaluation of the exact objective is intractable,
and neglects the assumption that we only have access to samples from $\Tpgam$
and are not able to evaluate it at arbitrary points.
\end{itemize}

We address each of these two issues in a principled manner.

%--------------------------------------------------------------------------
\vspace{-2mm}
\subsection{Eliminating degenerate solutions} 
\vspace{-1mm}
%--------------------------------------------------------------------------

To avoid degenerate solutions when $\gamma=1$,
we ensure that the solution is a proper density ratio; 
that is,
the property
$\tau\in \Xi\defeq\cbr{\tau\rbr{\cdot}\ge0, \EE_{p}\sbr{\tau} = 1}$
must be true of any $\tau$ that is a ratio of some density to $p$.
This provides an additional constraint that we add to the optimization
formulation
\begin{equation}
\label{eq:constrained_est}
\min_{\tau\ge 0}\,\,D\rbr{\Tpgam\circ\tau\|\ptau}, \quad\st,\quad \EE_{p}\sbr{\tau} = 1. 
\end{equation}
With this additional constraint, it is obvious that the trivial solution
$\tau\rbr{s, a} = 0$ is eliminated as an infeasible point of
\eqnref{eq:constrained_est}, along with other degenerate solutions $\tau\rbr{s, a} = c\tau^*\rbr{s, a}$ with $c\neq 1$.

Unfortunately,
exactly solving an optimization with expectation constraints 
is very complicated in general~\citep{LanZho16}, 
particularly given a nonlinear parameterization for $\tau$.
The penalty method~\citep{LueYe15} provides a much simpler alternative,
where a sequence of regularized problems are solved
%
%\vspace{-1mm}
\begin{equation}
\label{eq:regularized_est}
\min_{\tau\ge 0}\,\,J\rbr{\tau}\defeq D\rbr{\Tpgam\circ\tau\|\ptau} + \smallfrac{\lambda}{2} \rbr{\EE_{p}\sbr{\tau} - 1}^2,
\end{equation}
with $\lambda$ increasing.
The drawback of the penalty method is that it generally requires $\lambda \rightarrow\infty$ to ensure the strict feasibility, which is still impractical, 
\newadd{especially in stochastic gradient descent. The infinite $\lambda$ may induce \emph{unbounded variance} in the gradient estimator, and thus, \emph{divergence} in optimization.} However, by exploiting the special structure of the solution sets to~\eqref{eq:regularized_est}, we can show that, remarkably, it is unnecessary to increase $\lambda$.
\vspace{-2mm}
\begin{theorem}
\label{thm:soundness}
For $\gamma\in (0, 1]$ and any $\lambda >0$,
the solution to \eqref{eq:regularized_est}
is given by $\tau^*\rbr{s,a}=\frac{u\rbr{s, a}}{p\rbr{s, a}}$.
\end{theorem}
\vspace{-2mm}
The detailed proof for \thmref{thm:soundness} 
is given in~\appref{appendix:soundness}.
By~\thmref{thm:soundness}, we can estimate the desired 
correction ratio function $\tau^*$
by solving only one optimization with an arbitrary $\lambda>0$.

%--------------------------------------------------------------------------
\vspace{-2mm}
\subsection{Exploiting dual embedding}
\vspace{-1mm}
%--------------------------------------------------------------------------

The optimization in~\eqref{eq:regularized_est} 
involves the integrals $\Ttau$ and $\EE_p\sbr{\tau}$ 
inside nonlinear loss functions, hence appears difficult to solve.
Moreover, obtaining unbiased gradients with a naive approach
requires double sampling~\citep{Baird95}. 
Instead, we bypass both difficulties by applying a
dual embedding technique~\citep{DaiHePanBooetal16,DaiShaLiXiaHeetal17}. 
In particular, we assume the divergence $D$ is 
in the form of an $f$-divergence~\citep{nowozin2016f}
\[
\textstyle
D_\phi\rbr{\Ttau\|\ptau}\defeq \int\ptau\rbr{s, a}\phi\rbr{\frac{\Ttau\rbr{s, a}}{\ptau\rbr{s, a}}}ds\,da 
\]
where $\phi\rbr{\cdot}:\RR_+\rightarrow\RR$ is a convex,
lower-semicontinuous function with $\phi\rbr{1} = 0$.
Plugging this into $J\rbr{\tau}$ in~\eqref{eq:regularized_est}
we can easily check the convexity of the objective
\vspace{-2mm}
\begin{theorem}\label{thm:convexity}
For an $f$-divergence with valid $\phi$ defining $D_\phi$,
the objective $J\rbr{\tau}$ is convex w.r.t. $\tau$. 
\end{theorem}
\vspace{-1mm}
The detailed proof is provided in~\appref{appendix:convexity}.
Recall that a suitable convex function can be represented as
$\phi\rbr{x} = \max_{f} x\cdot f - \phi^*\rbr{f}$,
where $\phi^*$ is the Fenchel conjugate of $\phi\rbr{\cdot}$.
In particular, we have the representation
$\frac{1}{2}x^2 = \max_{u}ux - \frac{1}{2}u^2$,
which allows us to re-express the objective as
\begin{equation}
% \textstyle
\resizebox{0.92\hsize}{!}{$
J\rbr{\tau} = \int \ptau\rbr{s', a'}\cbr{\max_{f}\sbr{\frac{\Ttau\rbr{s', a'}}{\ptau\rbr{s', a'}}f - \phi^*\rbr{f}}}ds'da' + \lambda\cbr{\max_{u}\sbr{u\rbr{\EE_p\sbr{\tau}-1}-\frac{u^2}{2}}}.
$}
\end{equation}
Applying the interchangeability principle~\cite{ShaDen14,DaiHePanBooetal16},
one can replace the inner $\max$ in the first term over scalar $f$ to 
maximize over a function $f\rbr{\cdot, \cdot}:S\times A\rightarrow\RR$
\begin{multline}
\label{eq:saddle_est}
\min_{\tau\ge0}\max_{f:S\times A\rightarrow \RR, u\in\RR}\,\,J\rbr{\tau, u, f} = \rbr{1 - \gamma}\EE_{\mu_0\pi}\sbr{f\rbr{s, a}} + \gamma\EE_{\Tcal_p}\sbr{\tau\rbr{s, a}f\rbr{s', a'}} \\[-0.5ex]
- \EE_{p}\sbr{\tau\rbr{s, a}\phi^*\rbr{f\rbr{s, a}}} + \lambda\rbr{\EE_p\sbr{u\tau\rbr{s, a}-u}-\smallfrac{u^2}{2}}.
\end{multline}
This yields the main optimization formulation, which avoids the aforementioned difficulties and is well-suited for practical optimization as discussed in~\secref{sec:prac_alg}.

\vspace{-3mm}
\paragraph{Remark (Other divergences):}
\newadd{In addition to $f$-divergence, the proposed estimator~\eqref{eq:regularized_est} is compatible with other divergences,
such as the integral probability metrics~(IPM)~\citep{Mueller97,SriFukGreSchetal09}, while retaining consistency.
Based on the definition of the IPM, these divergences directly lead to $\min$-$\max$ optimizations similar to~\eqref{eq:saddle_est} with the identity function as $\phi^*\rbr{\cdot}$ and different feasible sets for the dual functions. Specifically, maximum mean discrepancy~(MMD)~\citep{SmoGreBor06} requires $\nbr{f}_{\Hcal_k}\le 1$ where $\Hcal_k$ denotes the RKHS with kernel $k$; the Dudley metric~\citep{Dudley02} requires $\nbr{f}_{BL}\le 1$ where $\nbr{f}_{BL}\defeq \nbr{f}_\infty + \nbr{\nabla f}_2$; and Wasserstein distance~\citep{ArjChiBot17} requires $\nbr{\nabla f}_2\le 1$.
% $\sup_{x, y\in \Omega}\cbr{\frac{\abr{f\rbr{x} - f\rbr{y}}}{\rho\rbr{x, y}}}$
These additional requirements on the dual function might incur some extra difficulty in practice. For example, with Wasserstein distance and the Dudley metric, we might need to include an extra gradient penalty~\citep{GulAhmArjDumetal17}, which requires additional computation to take the gradient through a gradient. Meanwhile, the consistency of the surrogate loss under regularization is not clear. For MMD, we can obtain a closed-form solution for the dual function, which saves the cost of the inner optimization~\cite{GreBorRasSchetal12}, but with the tradeoff of requiring \emph{two independent} samples in each outer optimization update. Moreover, MMD relies on the condition that the dual function lies in some RKHS, which introduces additional kernel parameters to be tuned and in practice may not be sufficiently flexible compared to neural networks.
}

%--------------------------------------------------------------------------
\vspace{-4mm}
\subsection{A Practical Algorithm}\label{sec:prac_alg}
\vspace{-1mm}
%--------------------------------------------------------------------------

We have derived a consistent stationary distribution correction estimator
in the form of a $\min$-$\max$ saddle point optimization~\eqref{eq:saddle_est}.
Here, we present a practical instantiation of~\estabb with 
a concrete objective and parametrization. 

We choose the $\chi^2$-divergence, which is an $f$-divergence with 
$\phi\rbr{x} = \rbr{x - 1}^2$ and
$\phi^*\rbr{y} = y + \frac{y^2}{4}$.
The objective becomes
\begin{multline}\label{eq:chi_saddle_est}
\textstyle
J_{\chi^2}\rbr{\tau, u, f} = \rbr{1 - \gamma}\EE_{\mu_0\pi}\sbr{f\rbr{s, a}} + \gamma\EE_{\Tcal_p}\sbr{\tau\rbr{s, a}f\rbr{s', a'}} \\
- \EE_{p}\sbr{\tau\rbr{s, a}\rbr{f\rbr{s, a} + \smallfrac{1}{4}f^2\rbr{s, a}}} + \lambda\rbr{\EE_p\sbr{u\tau\rbr{s, a}-u}-\smallfrac{u^2}{2}}.
\end{multline}
There two major reasons for adopting $\chi^2$-divergence: 
\begin{itemize}[leftmargin=*, nosep, topsep=0pt]
	\item[{\bf i)}]
In the behavior-agnostic OPE problem, we mainly use the 
ratio correction function for estimating
$\widehat\EE_{p}\sbr{\hat\tau\rbr{s, a}R\rbr{s, a}}$,
which is an expectation.
Recall that the error between the estimate
and ground-truth can then be bounded by total variation,
which is a lower bound of $\chi^2$-divergence. 
	\item[{\bf ii)}] 
For the alternative divergences, 
% the conjugate function of total variation is unbounded, 
the conjugate of the $KL$-divergence involves $\exp\rbr{\cdot}$, which may lead to instability in optimization; while the IPM variants introduce extra constraints on dual function, which may be difficult to be optimized. The conjugate function of $\chi^2$-divergence enjoys suitable numerical properties and provides squared regularization. \newadd{We have provided an empirical ablation study that investigates the alternative divergences in~Section \ref{subsec:ablation}.}
% \Roy{[Not sure, Bo may help on this] The divergence can also be generalized to integral probability measure (IPM). However, Wasserstein-1 needs the assumption of 1-Lipschitz; MMD divergence may avoid the adversarial training, but the kernel function may make the method computationally intensive~\cite{MMD_RL}.}
\end{itemize}
To parameterize the correction ratio $\tau$ and dual function $f$
we use neural networks,
$\tau\rbr{s, a} = \mathtt{nn}_{w_\tau}\rbr{s, a}$
and $f\rbr{s, a} = \mathtt{nn}_{w_f}\rbr{s, a}$,
where $w_\tau$ and $w_f$ denotes the parameters of $\tau$ and $f$ respectively.
Since the optimization requires $\tau$ to be non-negative,
we add an extra positive neuron, such as $\exp\rbr{\cdot}$,
$\log\rbr{1 + \exp\rbr{\cdot}}$ or $\rbr{\cdot}^2$ at the final layer of 
$\mathtt{nn}_{w_\tau}\rbr{s, a}$.
We empirically compare the different positive neurons 
in~\secref{subsec:ablation}. 

For these representations,
and unbiased gradient estimator 
$\nabla_{\rbr{w_\tau, u, w_f}} J\rbr{\tau, u, f}$ 
can be obtained straightforwardly,
as shown in~\appref{appendix:alg_details}.
This allows us to apply stochastic gradient descent to solve
the saddle-point problem~\eqref{eq:chi_saddle_est}
in a scalable manner,
as illustrated in~\Algref{alg:gendice}.

%!TEX root = dual_ratio.tex

%%%%%%%%%%%%%%%%%%%%%%%%%%%%%%%%%%%%%%%%%%%%%%%%%%%%%%%%%%%%%%%%%%%%%%%%%%%%
\vspace{-2mm}
\section{Theoretical Analysis}\label{sec:theoretical_analysis}
\vspace{-2mm}
%%%%%%%%%%%%%%%%%%%%%%%%%%%%%%%%%%%%%%%%%%%%%%%%%%%%%%%%%%%%%%%%%%%%%%%%%%%%

We provide a theoretical analysis for the proposed~\estabb algorithm,
following a similar learning setting and assumptions
to~\citep{nachum2019dualdice}. 
\vspace{-2mm}
\begin{assumption} %[Reference distribution property] % do we really need name?
\label{asmp:ref_dist}
% For any $\rbr{s, a}\in S\times A$, $\mu\rbr{s, a}>0$ implies $p\rbr{s, a}>0$. Furthermore, 
The target stationary correction are bounded, $\nbr{\tau^*}_\infty\le C<\infty$.
%\lihong{@Bo: the latter implies the former?  Also, for the informal statements below, we don't need the assumption here anyway?}\Bo{We need assumptions, I have the full version in~\appref{appendix:full_error}. I am thinking to put these assumptions somewhere in sec. 3.}
\end{assumption}
\vspace{-2mm}
%
% Basically, the assumption requires $\Tcal$ to be ergodic, \ie, aperiodic, Harris recurrent, and irreducible. For the details, please refer to~\citet{MeyTwe12}.
The main result is summarized in the following theorem.  A formal statement, together with the proof, is given in~\appref{appendix:proofs}.
\vspace{-3mm}
\begin{theorem}[Informal]
\label{thm:total_error}
Under mild conditions, with learnable $\Fcal$ and $\Hcal$,
the error in the objective between the~\estname estimate, ${\hat\tau}$,
to the solution $\tau^*\rbr{s, a} = \frac{u\rbr{s, a}}{p\rbr{s, a}}$
is bounded by
\vspace{-1mm}
$$
\textstyle
\EE\sbr{J\rbr{\hat\tau} - J\rbr{\tau^*}} =\widetilde\Ocal\rbr{\eps_{approx}\rbr{\Fcal, \Hcal} + {\frac{1}{\sqrt{N}}} + \eps_{opt}},
$$
where $\EE\sbr{\cdot}$ is w.r.t. the randomness in $\Dcal$ and in the optimization algorithms, $\eps_{opt}$ is the optimization error, and $\eps_{approx}\rbr{\Fcal, \Hcal}$ is the approximation induced by $\rbr{\Fcal, \Hcal}$ for parametrization of $\rbr{\tau, f}$.
\end{theorem}
\vspace{-2mm}
The theorem shows that the suboptimality of \estname's solution,
measured in terms of the objective function value,
can be decomposed into three terms:
(1) the approximation error $\eps_{approx}$,
which is controlled by the representation flexibility of function classes;
(2) the estimation error due to sample randomness,
which decays at the order of $1/\sqrt{N}$;
and
(3) the optimization error,
which arises from the suboptimality of the solution found by the
optimization algorithm.
As discussed in~\appref{appendix:proofs}, in special cases,
this suboptimality can be bounded below by a divergence between $\hat\tau$
and $\tau^*$,
and therefore directly bounds the error in the estimated policy value.

There is also a tradeoff between these three error terms.
With more flexible function classes (\eg, neural networks)
for $\Fcal$ and $\Hcal$, the approximation error $\eps_{approx}$
becomes smaller.
However, it may increase the estimation error 
(through the constant in front of $1/\sqrt{N}$)
and the optimization error (by solving a harder optimization problem).
On the other hand, if $\Fcal$ and $\Hcal$ are linearly parameterized,
estimation and optimization errors tend to be smaller and can often be
upper-bounded explicitly in~\appref{appendix:opt_error}.
However, the corresponding approximation error will be larger.

%!TEX root = dual_ratio.tex

%%%%%%%%%%%%%%%%%%%%%%%%%%%%%%%%%%%%%%%%%%%%%%%%%%%%%%%%%%%%%%%%%%%%%%%%%%%%
\vspace{-3mm}
\section{Related Work}\label{sec:related_work}
\vspace{-2mm}
%%%%%%%%%%%%%%%%%%%%%%%%%%%%%%%%%%%%%%%%%%%%%%%%%%%%%%%%%%%%%%%%%%%%%%%%%%%%
%\Bo{I went through my pass. Ruiyi please move the extra part into appendix.}
%\Lihong{I can go thru the OPE part.  We will also need MCMC and more general Markov chain stationary distribution related work?}

%\input{relatedwork.tex}

\paragraph{Off-policy Policy Evaluation}
Off-policy policy evaluation with %IS 
importance sampling (IS)
has has been explored in the contextual bandits \citep{strehl2010learning, dudik2011doubly,wang2017optimal}, and episodic RL settings~\citep{murphy2001marginal,precup01off}, achieving many empirical successes~\citep[e.g.,][]{strehl2010learning,dudik2011doubly,bottou13counterfactual}.
Unfortunately, IS-based methods suffer from exponential variance
in long-horizon problems,
known as the ``curse of horizon''~\citep{liu2018breaking}.
A few variance-reduction techniques have been introduced,
but still cannot eliminate this fundamental
issue~\citep{jiang2015doubly,thomas2016data,guo2017using}.
By rewriting the accumulated reward as an expectation w.r.t.\ 
a stationary distribution,~\citet{liu2018breaking,gelada2019off} recast OPE
as estimating a correction ratio function,
which significantly alleviates variance.
However, 
% COP-TD algorithm~\citep{gelada2019off} is the first work that attempts to solve this problem via estimating the density ratio of two stationary state distributions, with trajectories only from some behavior distribution. But it can only be applied with discrete state/action spaces. 
% To extend to continuous case, \cite{liu2018breaking} developed a mini-max loss function for the stationary state distributions estimation, and derive a closed-form solution for the case of RKHS. 
these methods still require the off-policy data to be collected by a \emph{single
and known} behavior policy,
which restricts their practical applicability.
The only published algorithm in the literature, to the best of our knowledge, that solves agnostic-behavior off-policy evaluation is DualDICE~\citep{nachum2019dualdice}.
However, DualDICE was developed for discounted problems
and its results become unstable when the discount factor approaches $1$
(see below).
By contrast, \estname can cope with the more challenging problem of
undiscounted reward estimation in the general behavior-agnostic setting. 
 
Note that standard model-based methods~\citep{SutBar98},
which estimate the transition and reward models directly
then calculate the expected reward based on the learned model,
are also applicable to the behavior-agnostic setting considered here.
Unfortunately, model-based methods typically rely heavily
on modeling assumptions about rewards and transition dynamics.
In practice, these assumptions do not always hold,
and the evaluation results can become unreliable.
%usually fail to capture the complex environments and reward, and may induce extra bias from the simulation. 

% For more related work on MCMC, density ratio estimation and PageRank, please refer to~\appref{appendix:more_related_work}.
\vspace{-3mm}
\paragraph{Markov Chain Monte Carlo}
Classical MCMC~\citep{brooks2011handbook, gelman2013bayesian} aims at sampling from $\mu^\pi$ by iteratively simulting from the transition operator. 
% based on a given unnormalized form in physics, statistics and machine learning. 
% Similar methods include variational Bayes~\citep{hoffman2013stochastic, kingma2013auto} and expectation propagation~\citep{minka2001expectation}. 
% Particle-based methods avoid the parametric assumptions on the $\mu^\pi$, such as SVGD, but usually render heavy computational cost to update many particles.
It requires continuous interaction with the transition operator and heavy computational cost to update many particles. Amortized SVGD~\citep{wang2016learning} and Adversarial MCMC~\citep{song2017nice, li2019adversarial} alleviate this issue via combining with neural network, but they still 
% aim at approximating a distribution (learning a sampler) 
interact with the transition operator directly, \ie, in an on-policy setting. The major difference of our~\estname is the learning setting: we only access the  off-policy dataset, and cannot sample from the transition operator. The proposed \estname leverages stationary density ratio estimation for approximating the stationary quantities, 
%Our proposed \estname is derived by exploiting the Fenchel duality of the $f$-divergences, but the $f$-divergence is built upon the property of steady state of a Markov chain. Furthermore, we only require the data distribution, \ie, off-policy dataset, and do not need to sample from our model. Thus \estname uses a very general variational minimization technique to estimate the stationary ratio of a Markov chain, instead of approximate the stationary distribution, 
which distinct it from classical methods.
\vspace{-3mm}
\paragraph{Density Ratio Estimation}  
%\Bo{The references listed here are not quite related. I would use this section for dual embedding and GAN. It seems we do not have enough space. Put this into the Appendix.}
%\Roy{TODO: adding more calssical methods as DualDICE.}
Density ratio estimation is a fundamental tool in machine learning and much related work exists. Classical density ratio estimation includes moment matching \citep{gretton2009covariate}, probabilistic classification~\citep{bickel2007discriminative}, and ratio matching~\citep{NguWaiJor08,SugNakKasBueetal08,kanamori2009least}. 
% Recent progress shows the power of adversarial training on realistic image generation~\citep{goodfellow2014generative}, text generation~\citep{yu2017seqgan} and even imitation learning~\citep{ho2016generative}. 
These classical methods focus on estimating the ratio between two distributions with samples from both of them, while \estname estimates the density ratio to a stationary distribution of a transition operator, from which even one sample is difficult to obtain.
\vspace{-3mm}
\paragraph{PageRank} \citet{yao13reinforcement} developed a reverse-time RL framework for PageRank via solving a reverse Bellman equation, which is
less sensitive to graph topology and shows faster adaptation with graph change. However, \citet{yao13reinforcement} still considers the online manner, which is different with our OPR setting.

\vspace{-3mm}
\section{Experiments}\label{sec:experiments}
\vspace{-3mm}
%%%%%%%%%%%%%%%%%%%%%%%%%%%%%%%%%%%%%%%%%%%%%%%%%%%%%%%%%%%%%%%%%%%%%%%%%%%%

%\input{experiment.tex}

%\Bo{@Ruiyi, please also complete the extra experiment part in Appendix and referred in main text somewhere. }
In this section, we evaluate \estname on OPE and OPR problems.
For OPE, we use one or multiple behavior policies to collect a fixed number
of trajectories at some fixed trajectory length.
This data is used to recover a correction ratio function for a
target policy $\pi$ that is then used to estimate the average reward
in two different settings:
\RN{1}) average reward; and \RN{2}) discounted reward.
In both settings, we compare with a model-based approach and
step-wise weighted IS~\citep{PreSutSin00}. \newadd{We also compare to 
% the estimator of 
\citet{liu2018breaking} (referred to as ``IPS'' here) in the Taxi domain with a learned behavior policy\footnote{We used the released implementation of IPS~\citep{liu2018breaking} from {\url{https://github.com/zt95/infinite-horizon-off-policy-estimation}}.}.} We specifically compare to DualDICE~\citep{nachum2019dualdice} in the discounted reward setting, which is a direct and current state-of-the-art baseline. For OPR, the main comparison is with the model-based method,
%we focus on comparing with the model-based method,
where the transition operator is empirically estimated
and stationary distribution recovered via an exact solver.
We validate \estname in both tabular and continuous cases,
and perform an ablation study to further demonstrate its effectiveness.
All results are based on 20 random seeds,
with mean and standard deviation plotted. 
Our code is publicly available at \url{https://github.com/zhangry868/GenDICE}.

%---------------------------------------------------------------------------
\vspace{-2mm}
\subsection{Tabular Case}\label{exp:tabular}
\vspace{-2mm}
%---------------------------------------------------------------------------
\begin{wrapfigure}{R}{0.39\linewidth}
	\vspace{-5mm}
	\centering
	\includegraphics[width=\linewidth]{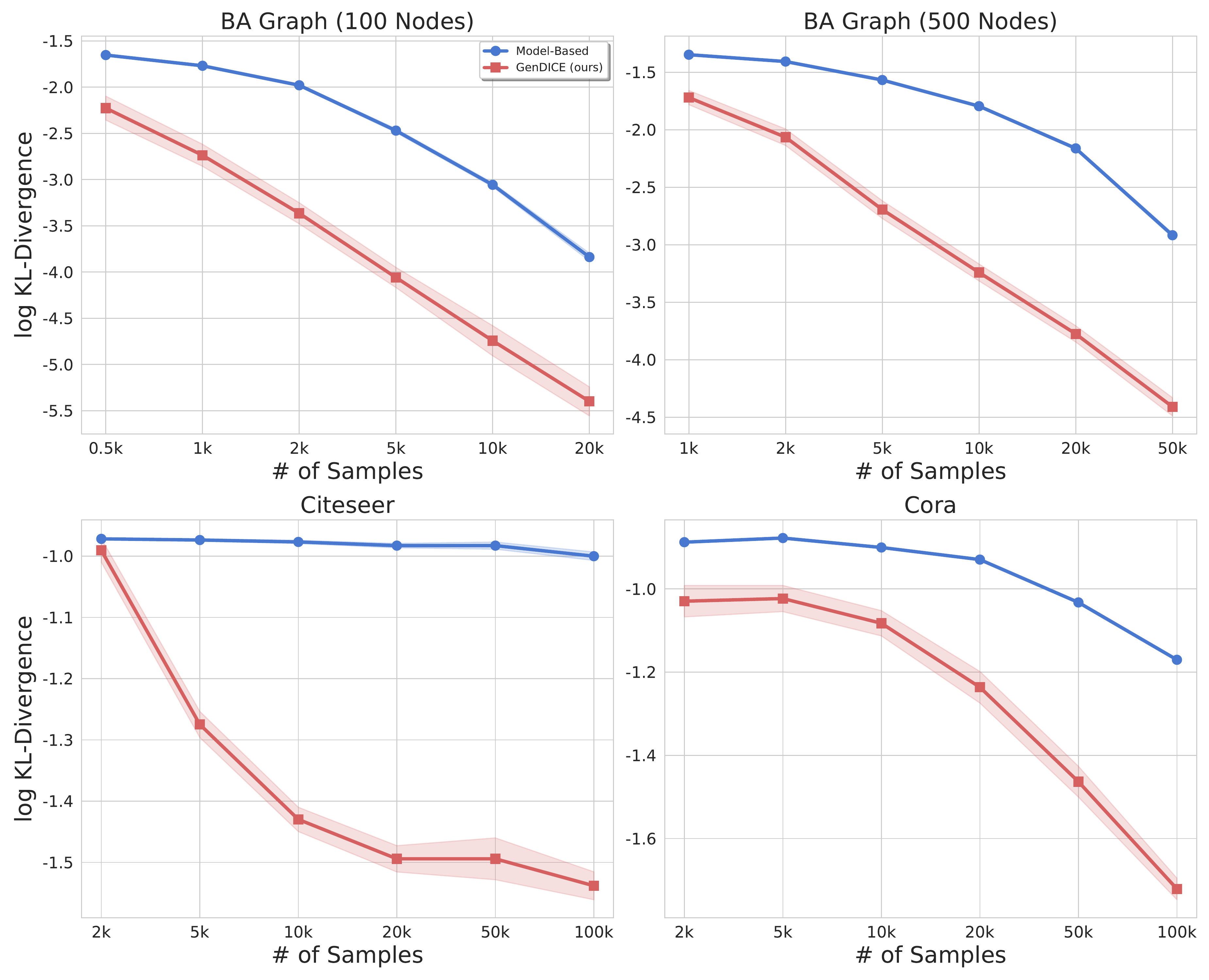}\\
	\vspace{-4mm}
	\caption{Stationary Distribution Estimation on BA and real-world graphs. Each plot shows the $\log$ $KL$-divergence of \estname and model-based method towards the number of samples.}
	\label{fig:offpolicy-graph}
	\vspace{-3mm}
\end{wrapfigure}
\paragraph{Offline PageRank on Graphs}
One direct application of \estname is off-line PageRank (OPR).
% \vspace{-3mm}
% \paragraph{Barabasi-Albert random graphs}
We test \estname on a Barabasi-Albert (BA) graph (synthetic), and two real-world graphs, Cora and Citeseer.
%, which is a random scale-free synthetic network. % using a preferential attachment mechanism (). 
%We have 100 and 500 nodes, with the minimum degree as 4. 
Details of the graphs are given in Appendix \ref{appendix:exp_settings}. 
We use the $\log$ $KL$-divergence between estimated stationary distribution and the ground truth as the evaluation metric,  with the ground truth computed by an exact solver based on the exact transition operator of the graphs.
%
% \vspace{-3mm}
% \paragraph{Real-world graphs}
%
% 
We compared \estname with model-based methods in terms of the sample efficiency. From the results in \figref{fig:offpolicy-graph}, \estname outperforms the model-based method when limited data is %off-policy samples are 
given.
Even with $20k$ samples for a BA graph with $100$ nodes, where a transition matrix has $10k$ entries, \estname still shows better performance in the offline setting. This is reasonable since \estname directly estimates the stationary distribution vector or ratio, while the model-based method needs to learn an entire transition matrix that has many more parameters.
%, then computes the stationary distribution via an explicit solver. %solving the estimated transition matrix. 

\vspace{-3mm}
\paragraph{Off-Policy Evaluation with Taxi}
We use a similar taxi domain as in \citet{liu2018breaking},
where a grid size of $5 \times 5$ yields $2000$ states in total ($25\times16\times5$, corresponding to $25$ taxi locations, 16 passenger appearance status and $5$ taxi status).
We set the target policy to a final policy $\pi$ after running tabular Q-learning for $1000$ iterations, and set another policy $\pi_+$ after $950$ iterations as the base policy. The behavior policy is a mixture controlled by $\alpha$ as 
$\pi_b = (1-\alpha)\pi + \alpha \pi_+$.
% \ie, the larger $\alpha$ is, the behavior policy is more close to the target policy. 
For the model-based method,
we use a tabular representation for the reward and transition functions, whose entries are estimated from behavior data.
%we use a matrix as the approximation of the reward function and another matrix as the approximation of the environment dynamics.
For IS and IPS, we fit a policy via behavior cloning to estimate the policy ratio. 
%with the collected trajectories of unknown behavior policies, and use it to estimate the policy ratio, \ie, $\pi/\pi_b$. 
In this specific setting, our methods achieve better results compared to IS, IPS and the model-based method.  Interestingly, with longer horizons, IS cannot improve as much as other methods even with more data, while \estname consistently improve and achieves much better results than the baselines. DualDICE only works with $\gamma<1$.  %and as we vary $\gamma$. 
\estname is more stable than DualDICE when $\gamma$ becomes larger (close to $1$), while still showing competitive performance for smaller discount factors $\gamma$.

\begin{figure}[ht] \centering
	\vspace{-3mm}
	\includegraphics[width=\linewidth]{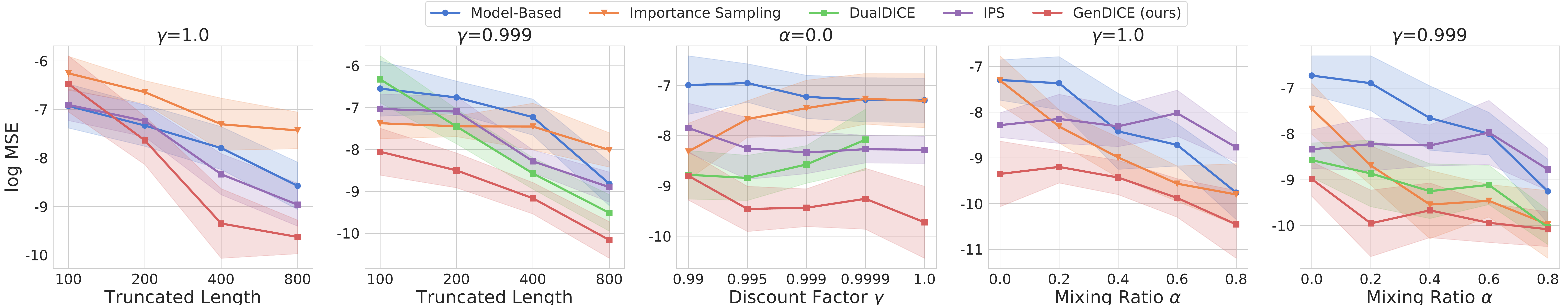}
	\vspace{-6mm}
	\caption{Results on Taxi Domain. The plots show log MSE of the tabular estimator across different trajectory lengths, different discount factors and different behavior policies ($x$-axis).}
	\label{fig:offpolicy-taxi}
	% \vspace{1mm}
\end{figure}
%---------------------------------------------------------------------------
% \vspace{1mm}
\subsection{Continuous Case}
\vspace{-1mm}
%---------------------------------------------------------------------------
% \Bo{There is no need to cite DualDICE everywhere.}
We further test our method for OPE on three control tasks: a discrete-control task Cartpole and two continuous-control
tasks Reacher and HalfCheetah. In these tasks, observations (or states) are continuous, thus we use neural network function
approximators and stochastic optimization. Since DualDICE~\citep{nachum2019dualdice} has shown the state-of-the-art performance on discounted OPE, we mainly compare with it in the discounted reward case. We also compare to IS with a learned policy via behavior cloning and a neural model-based method, similar to the tabular case, but with neural network as the function approximator. All neural networks are feed-forward with two hidden layers of dimension $64$ and $\tanh$ activations. More details can be found in Appendix \ref{appendix:exp_settings}.

%
%\paragraph{OPE for Continuous Control}
Due to limited space, we put the discrete control results in
 Appendix \ref{appendix:exp} 
and focus on the more challenging continuous control tasks.
Here, the good performance of IS and model-based methods in
Section \ref{exp:tabular} quickly deteriorates as the environment
becomes complex, \ie, with a continuous action space. 
Note that \estname is able to maintain good performance in this scenario,
even when using function approximation and stochastic optimization.
This is reasonable because of the difficulty of fitting to the coupled policy-environment dynamics with a continuous action space.
Here we also \textit{empirically} validate \estname with off-policy data collected by multiple policies. 

As illustrated in Figure \ref{fig:offpolicy-cartpole}, all methods perform better with longer trajectory length or more trajectories. When $\alpha$ becomes larger, \ie, the behavior policies are closer to the target policy, all methods performs better, as expected.  Here, \estname demonstrates good performance both on average-reward and discounted reward cases in different settings. The right two figures in each row show the $\log$ MSE curve versus optimization steps, where \estname achieves the smallest loss.
In the discounted reward case, \estname shows significantly better and more stable performance than the strong baseline, DualDICE. Figure \ref{fig:offpolicy-half} also shows better performance of \estname than all baselines in the more challenging HalfCheetah domain.
%\footnote{We found DualDICE is sensitive to the trajectory length, and the settings are slightly different from that in \cite{nachum2019dualdice}.}.
\begin{figure}[ht] \centering
	\vspace{-3mm}
	\includegraphics[width=1\linewidth]{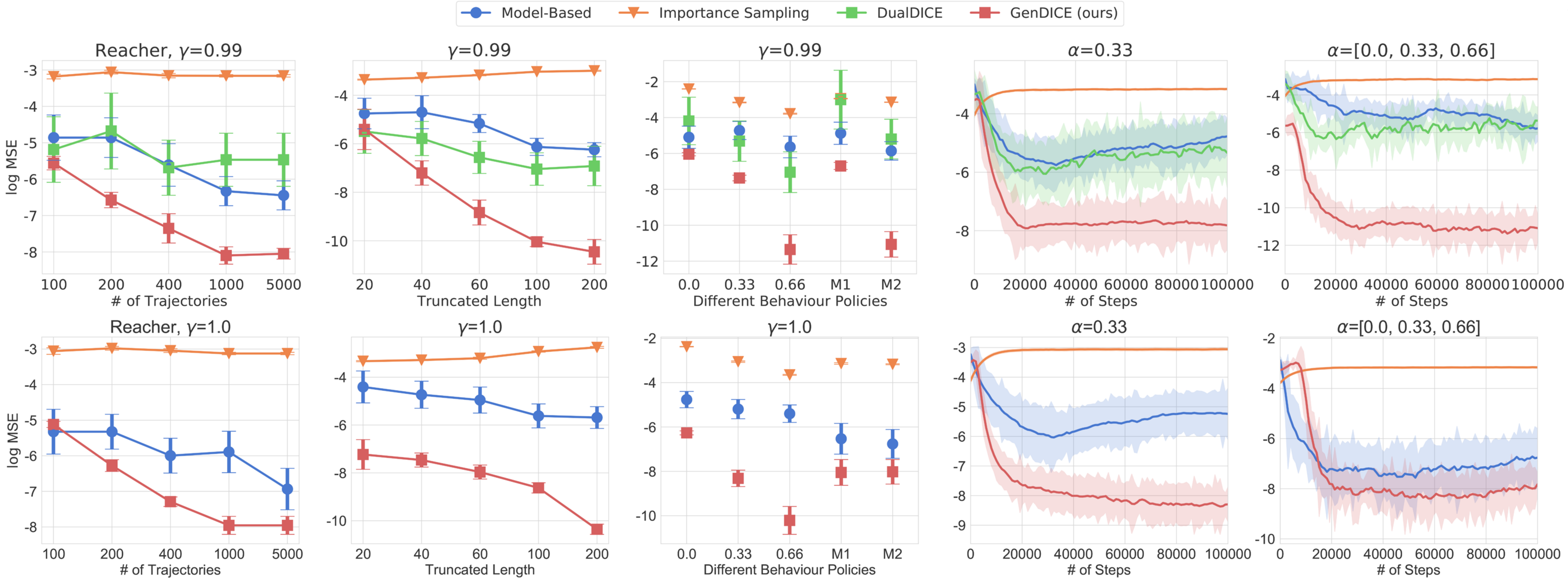}
	\vspace{-6mm}
	\caption{Results on Reacher. The left three plots in the first row show the $\log$ MSE of estimated average per-step reward over different numbers of trajectories, truncated lengths, and behavior policies (M1 and M2 mean off-policy set collected by multiple behavior policies with $\alpha=[0.0, 0.33]$ and $\alpha=[0.0, 0.33, 0.66]$). The right two figures show the loss curves towards the optimization steps. Each plot in the second row shows the average reward case.}
	\label{fig:offpolicy-cartpole}
\end{figure}
\vspace{-6mm}
\begin{figure}[ht] \centering
	\includegraphics[width=1\linewidth]{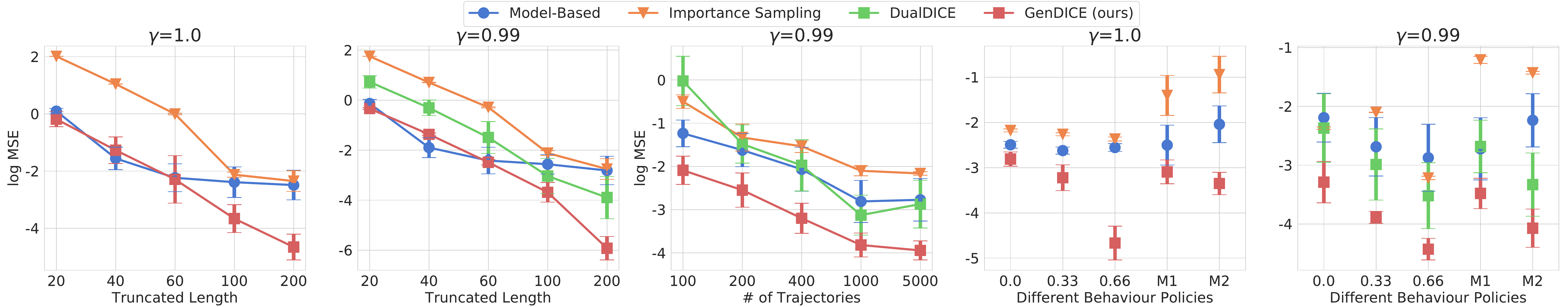}
	\vspace{-7mm}
%	\caption{Results on HalfCheetah. Plots from left to the right show the $\log$ MSE of estimated average per-step reward over different number of trajectories, truncated lengths, and behavior policies in discounted and average reward cases. }
	\caption{Results on HalfCheetah. Plots from left to the right show the $\log$ MSE of estimated average per-step reward over different truncated lengths, numbers of trajectories, and behavior policies in discounted and average reward cases. }
	\label{fig:offpolicy-half}
	\vspace{-3mm}
\end{figure}

%---------------------------------------------------------------------------
\vspace{-3mm}
\subsection{Ablation Study}\label{subsec:ablation}
\vspace{-2mm}
%---------------------------------------------------------------------------
Finally, we conduct an ablation study on \estname to study its robustness and implementation sensitivities. We investigate the effects of learning rate, activation function, discount factor, and the specifically designed ratio constraint. We further demonstrate the effect of the choice of divergences and the penalty weight. 

\vspace{-3mm}
\paragraph{Effects of the Learning Rate}
Since we are using neural network as the function
approximator, and stochastic optimization,
it is necessary to show sensitivity to the learning rate with $\{0.0001, 0.0003, 0.001, 0.003\}$, with results in Figure \ref{fig:ablationlr}. When $\alpha=0.33$, \ie, the OPE tasks are relatively easier and \estname obtains better
results at all learning rate settings.
However, when $\alpha=0.0$, \ie, the estimation becomes more difficult 
and only \estname only obtains reasonable results with the larger learning rate.
Generally, this ablation study shows that the proposed method is not
sensitive to the learning rate, and is easy to train. 
%More results are shown in Appendix \ref{appendix:exp}.
\vspace{-3mm}
\paragraph{Activation Function of Ratio Estimator}
We further investigate the effects of the activation function on the last layer, which ensure the non-negative outputs required for the ratio.
%How to ensure the non-negative outputs of this estimator is important, and we put an activation function on our last layer. 
To better understand which activation function will lead to stable trainig for the neural correction estimator, we empirically compare using \RN{1}) $(\cdot)^2$; \RN{2}) $\log(1+\exp(\cdot))$; and \RN{3}) $\exp(\cdot)$. %The first two activation functions work well; while $\exp(\cdot)$ shows poor performance under some settings. 
In practice, we use the $(\cdot)^2$ since it achieves low variance and better performance in most cases, as shown in Figure \ref{fig:ablationlr}.
\vspace{-3mm}
\paragraph{Effects of Discount Factors} 
We vary $\gamma\in\{0.95, 0.99, 0.995, 0.999, 1.0\}$ to probe the sensitivity of \estname. Specifically, we compare to DualDICE, and find that \estname is stable, while DualDICE becomes unstable when the $\gamma$ becomes large,
as shown in Figure \ref{fig:ablationdis}. 
\estname is also more general than DualDICE, as it can be applied to both the average and discounted reward cases.

\vspace{-3mm}
\paragraph{Effects of Ratio Constraint}
In Section \ref{sec:dual_est}, we highlighted the importance of 
the ratio constraint.
Here we investigate the trivial solution issue without the constraint.
The results in Figure~\ref{fig:ablationdis} demonstrate the necessity of
adding the constraint penalty,
since a trivial solution prevents an accurate corrector from being
recovered (green line in left two figures). 

\begin{figure}[ht] \centering
	\vspace{-1mm}
	\begin{tabular}{cc}
		\hspace{-3mm}
		\includegraphics[width=0.41\linewidth]{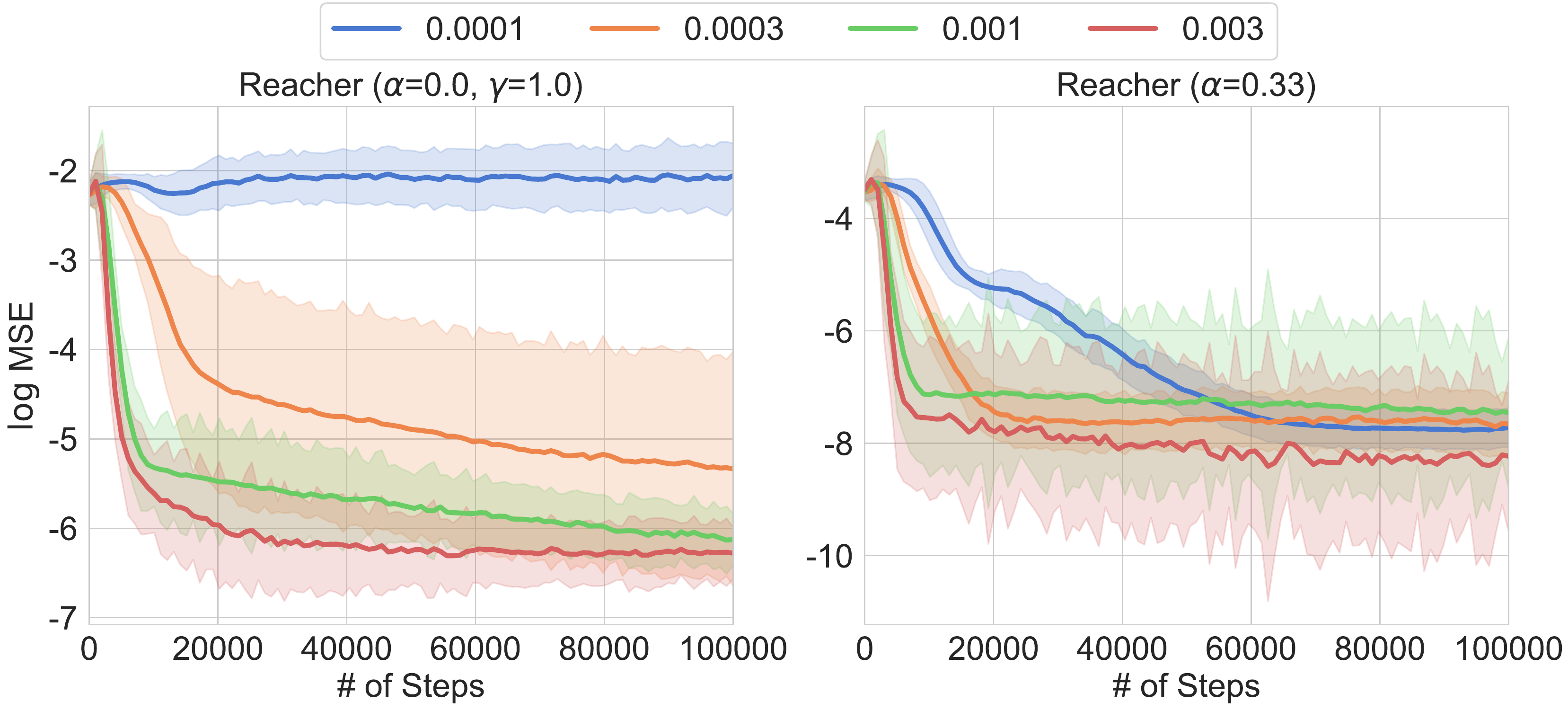}  &
		\hspace{-6mm}
		\includegraphics[width=0.6\linewidth]{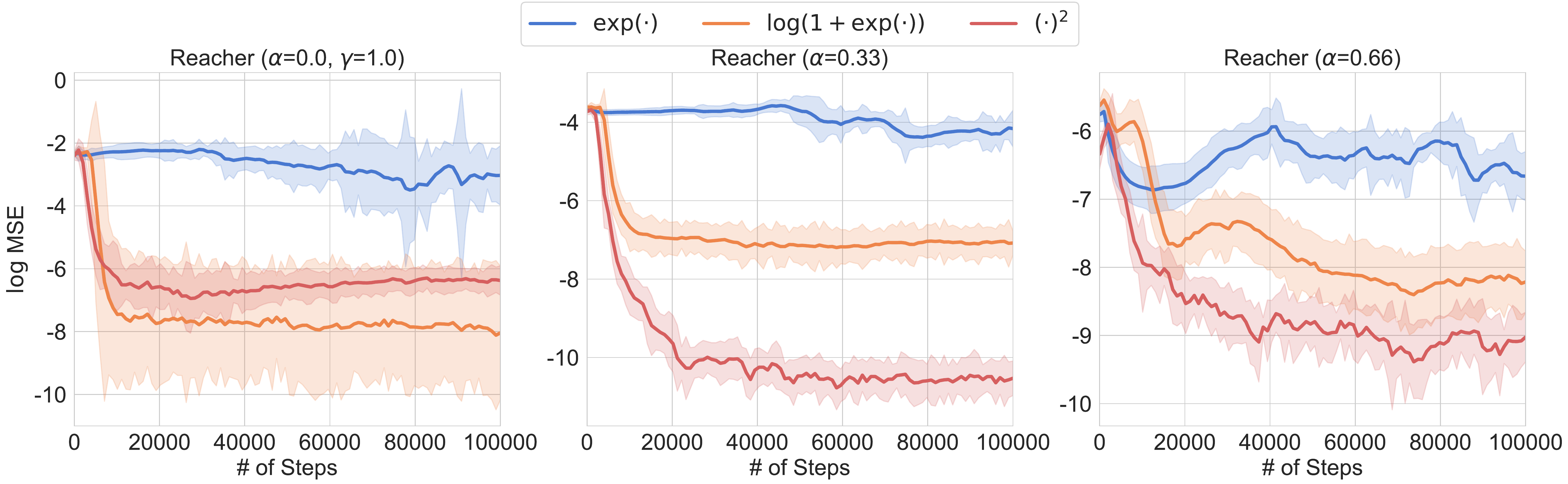}
		\\
	\end{tabular} 
	\vspace{-4mm}
	\caption{{Results of ablation study with different learning rates and activation functions. The plots show the log MSE of estimated average per-step reward over training and different behavior policies.}}
	\label{fig:ablationlr}
\end{figure}
\begin{figure}[ht] \centering
	\begin{tabular}{cc}
		\hspace{-3mm}
		\includegraphics[width=0.4\linewidth]{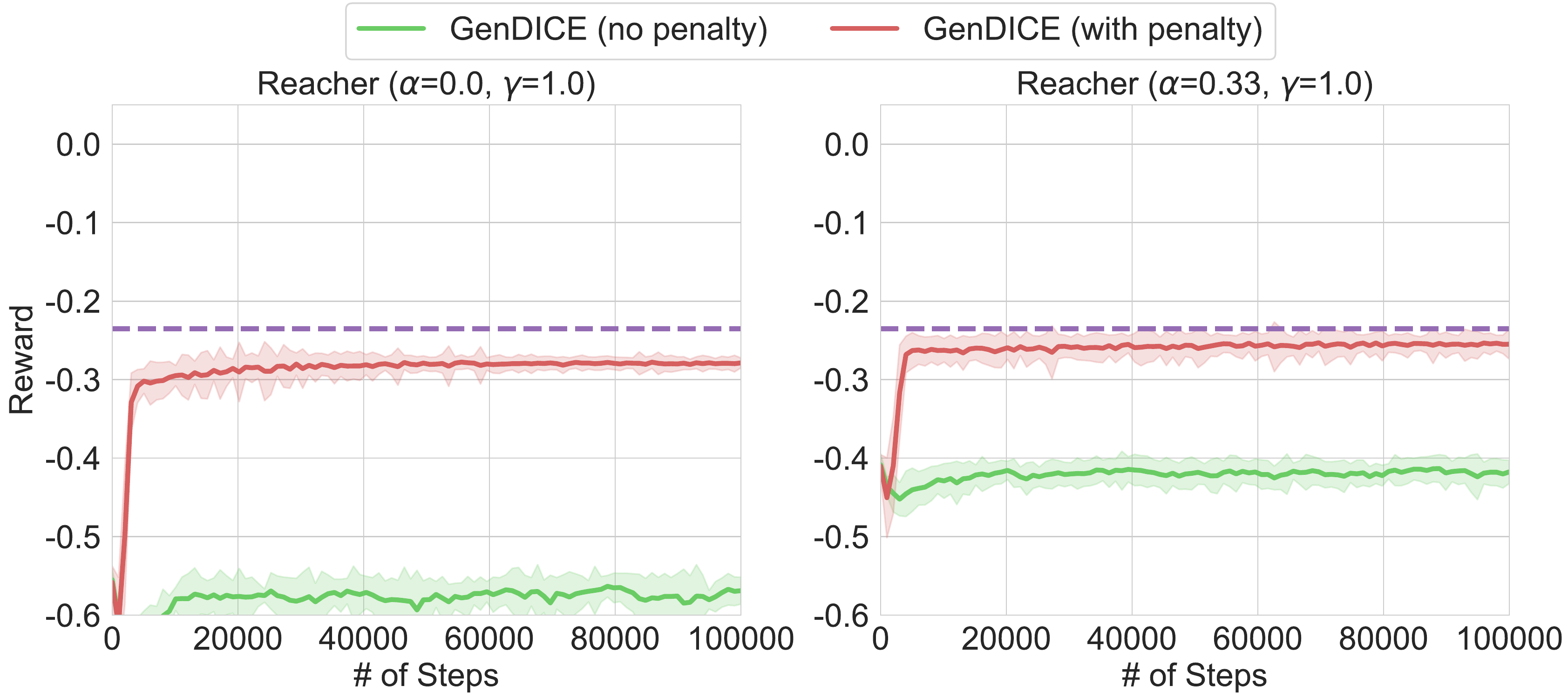}  &
		\hspace{-5mm}
		\includegraphics[width=0.6\linewidth]{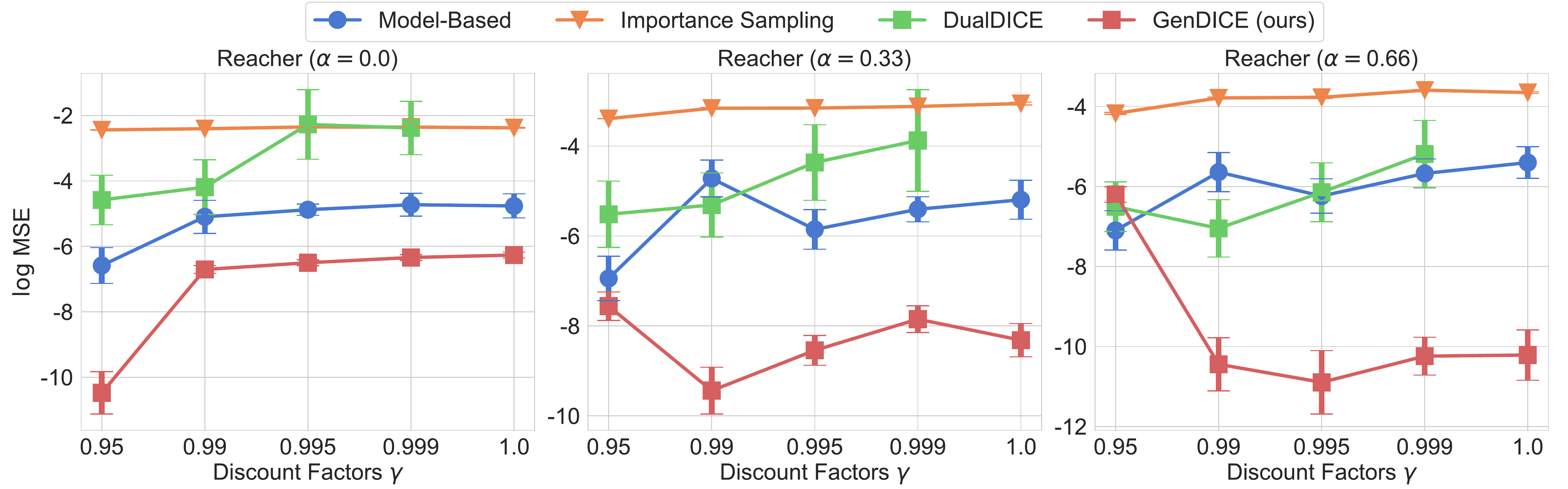}
		\\
	\end{tabular} 
	\vspace{-4mm}
	\caption{{Results of ablation study with constraint penalty and discount factors. The left two figures show the effect of ratio constraint on estimating average per-step reward. The right three figures show the $\log$ MSE for  average per-step reward over training and different discount factor $\gamma$.}}
	\label{fig:ablationdis}
\end{figure}

\begin{wrapfigure}{R}{0.39\linewidth} \centering
	\begin{tabular}{cc}
		\hspace{-3mm}
		\includegraphics[width=0.48\linewidth]{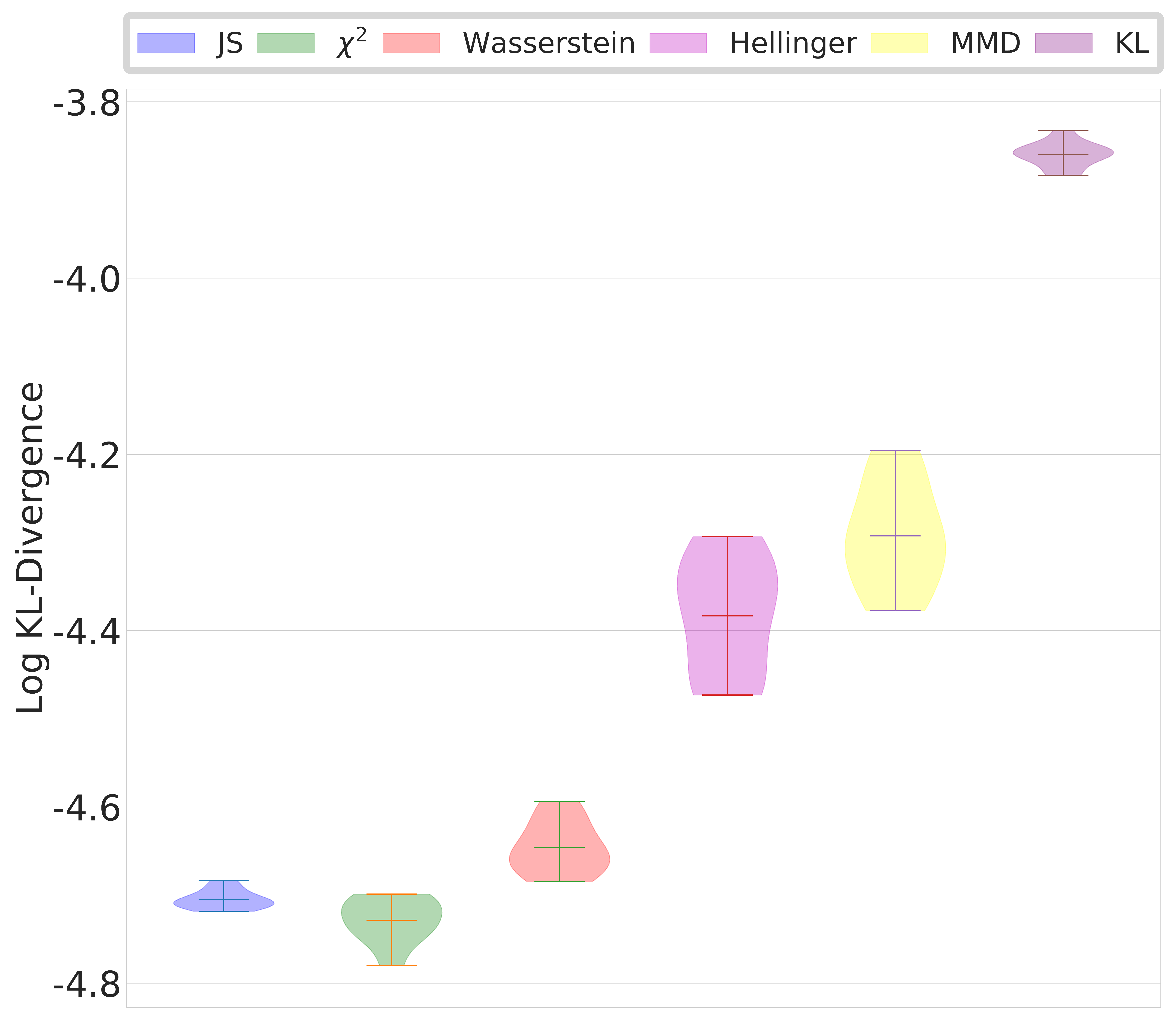}  &
		\hspace{-4mm}
		\includegraphics[width=0.5\linewidth]{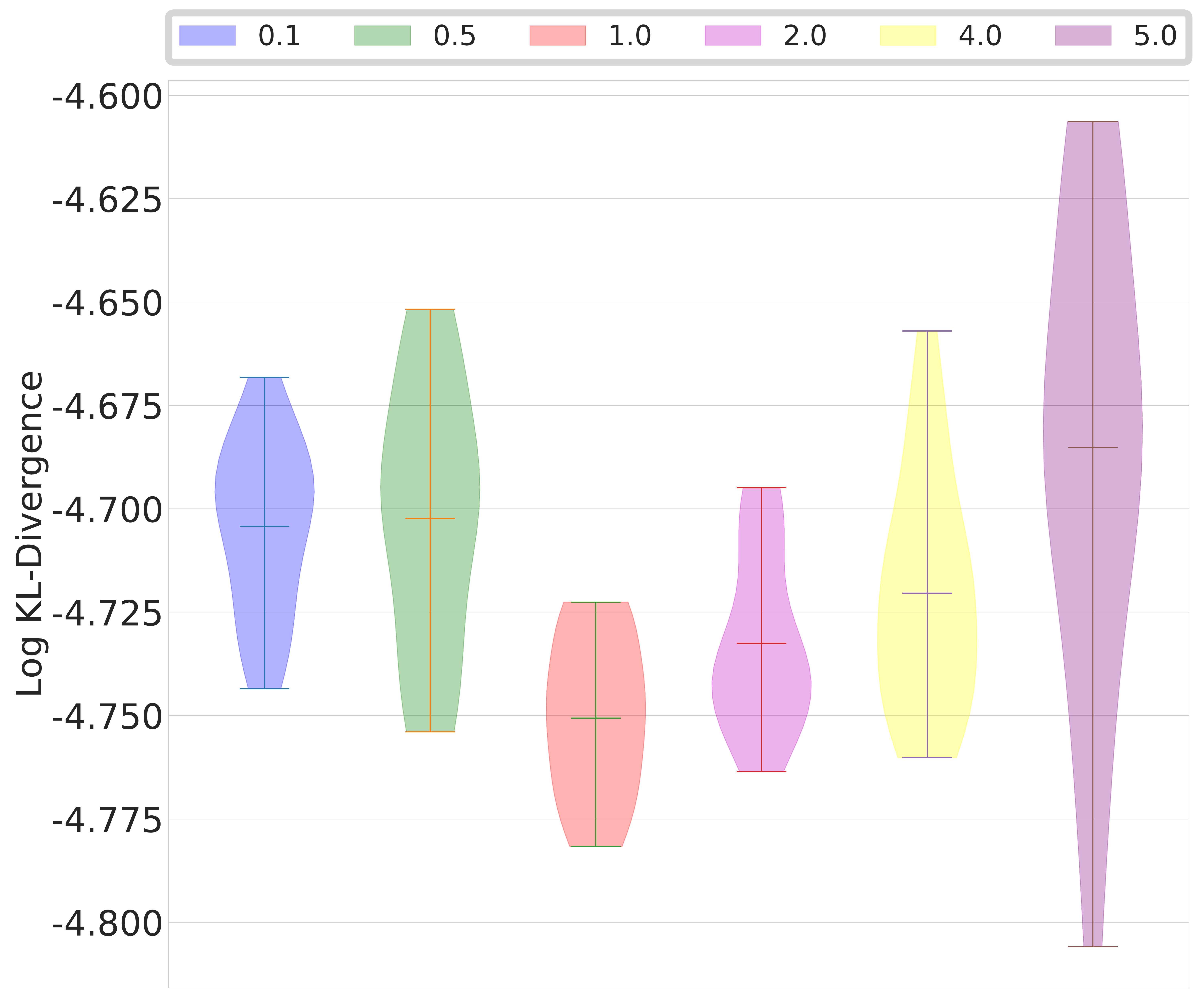}\\
		\vspace{-4mm}
		\\
		(a) &
		(b) \\
	\end{tabular} 
	\vspace{-3mm}
	\caption{{Results of ablation study with (a) different divergence and (b) weight of penalty $\lambda$. The plots show the log $KL$-Divergence of OPR on Barabasi-Albert graph.}}
	\label{fig:opr_ablation_main}
	\vspace{-3mm}
\end{wrapfigure}

\vspace{-3mm}
\paragraph{Effects of the Choice of Divergences} We empirically test the GenDICE with several other alternative divergences, \eg, Wasserstein-$1$ distance, Jensen-Shannon divergence, $KL$-divergence, Hellinger divergence, and MMD. 
To avoid the effects of other factors in the estimator, \eg, function parametrization, we focus on the offline PageRank task on BA graph with $100$ nodes and $10k$ offline samples. All the experiments are evaluated with $20$ random trials. 
To ensure the dual function to be $1$-Lipchitz, we add the gradient penalty. Besides, we use a learned Gaussian kernel in MMD, similar to~\citet{LiChaYuYanetal17}. As we can see in~\figref{fig:opr_ablation_main}(a), the GenDICE estimator is compatible with many different divergences. Most of the divergences, with appropriate extra techniques to handle the difficulties in optimization and carefully tuning for extra parameters, can achieve similar performances, consistent with phenomena in the variants of GANs~\citep{LucKurMicGeletal18}. However, $KL$-divergence is an outlier, performing noticeably worse, which might be caused by the ill-behaved $\exp\rbr{\cdot}$ in its conjugate function. The $\chi^2$-divergence and JS-divergence are better, which achieve good performances with fewer parameters to be tuned.

\vspace{-3mm}
\paragraph{Effects of the Penalty Weight} The results of different penalty weights $\lambda$ are illustrated in~\figref{fig:opr_ablation_main}(b). We vary the $\lambda \in [0.1, 5]$ with $\chi^2$-divergence. Within a large range of $\lambda$, the performances of the proposed GenDICE are quite consistent, which justifies~\thmref{thm:soundness}. 
The penalty multiplies with $\lambda$. Therefore, with $\lambda$ increases, the variance of the stochastic gradient estimator also increases, which explains the variance increasing in large $\lambda$ in~\figref{fig:opr_ablation_main}(b). In practice, $\lambda = 1$ is a reasonable choice for general cases.
\vspace{-3mm}

%%%%%%%%%%%%%%%%%%%%%%%%%%%%%%%%%%%%%%%%%%%%%%%%%%%%%%%%%%%%%%%%%%%%%%%%%%%%
\section{Conclusion}\label{sec:conclusion}
%%%%%%%%%%%%%%%%%%%%%%%%%%%%%%%%%%%%%%%%%%%%%%%%%%%%%%%%%%%%%%%%%%%%%%%%%%%%

In this paper, we proposed a novel algorithm \estname for general stationary distribution correction estimation, which can handle both the discounted and average stationary distribution given multiple behavior-agnostic samples.
% of a Markov chain, without any unnecessary assumptions on the offline data. 
Empirical results on off-policy evaluation and offline PageRank show the superiority of proposed method over the existing state-of-the-art methods.

% compared with DualDICE which can only be applied in discounted setting, and IS-based method~\citep{liu2018breaking} with assumptions restricting its application.  
% Future direction includes applying \estname for imitation learning to improve data efficiency, and incorporating it in policy learning (e.g., DQN). Furthermore, it can be applied in real-world applications, such as off-policy recommendation system~\citep{chen2019top} and navigation tasks~\citep{anderson2018vision}.  

\subsubsection*{Acknowledgments}
The authors would like to thank Ofir Nachum, the rest of the Google Brain team and the anonymous reviewers for helpful discussions and feedback.


%%%%%%%%%%%%%%%%%%%%%%%%%%%%%%%%%%%%%%%%%%%%%%%%%%%%%%%%%%%%
%%%% Reference
%%%%%%%%%%%%%%%%%%%%%%%%%%%%%%%%%%%%%%%%%%%%%%%%%%%%%%%%%%%%
% \bibliographystyle{unsrt}
% \bibliographystyle{plainnat}
% \clearpage
% \newpage

\bibliographystyle{iclr2020_conference}
% \bibliography{iclr2020_conference,../../../../bibfile/bibfile}
\bibliography{iclr2020_conference,../../bibfile/bibfile}

%----------------------------------------------------------------------------------------------------------------------------------
%----------------------------------------------------------------------------------------------------------------------------------
\clearpage
\newpage

\appendix
\onecolumn

\begin{appendix}
	
\thispagestyle{plain}
\begin{center}
	{\huge Appendix}
\end{center}

\section{Properties of~\estabb}\label{appendix:properties}
%%%%%%%%%%%%%%%%%%%%%%%%%%%%%%%%%%%%%%%%%%%%%%%%%%%%%%%%%%%%%%%%%%%%%%%%%%%%

For notation simplicity, we denote $x = \rbr{s, a} \in \Omega \defeq S\times A$ and $\Pb^\pi\rbr{x'|x} \defeq \pi\rbr{a'|s'}\Pb\rbr{s'|s, a}$. Also define $\nbr{f}_{p, 2} \defeq \inner{f}{f}_p = \int f\rbr{x}^2p\rbr{x}dx$. We make the following assumption to ensure the existence of the stationary distribution. Our discussion is all based on this assumption. 
% \begin{assumption}[Markov chain regularity]\label{asmp:stat_exist}
\paragraph{\asmpref{asmp:stat_exist}}
\textit{Under the target policy, the resulted state-action transition operator $\Tcal$ has a unique stationary distribution in terms of the divergence $D\rbr{\cdot||\cdot}$.}
% \end{assumption}

If the total variation divergence is selected, the~\asmpref{asmp:stat_exist} requires the transition operator should be ergodic, as discussed in~\citet{MeyTwe12}. 
%\lihong{Why need divergence and/or TV conditions?}
%\lihong{Should also move the assumption to main text.}

%---------------------------------------------------------------------------
\subsection{Consistency of the Estimator}\label{appendix:soundness}
%---------------------------------------------------------------------------

% \lihong{``Unbiasedness'' to ``Soundness''?}
\paragraph{\thmref{thm:soundness}}
\textit{For arbitrary $\lambda >0$, the solution to the optimization~\eqnref{eq:regularized_est} is $\frac{u\rbr{s, a}}{p\rbr{s, a}}$ for $\gamma\in (0, 1]$. }
\begin{proof}
For $\gamma \in(0, 1)$, there is not degenerate solutions to $D\rbr{\Ttau||\ptau}$. The optimal solution is a density ratio. Therefore, the extra penalty $\rbr{\EE_{p\rbr{x}}\sbr{\tau\rbr{x}} - 1}^2$ does not affect the optimality for $\forall\lambda>0$. 

When $\gamma =1$, for $\forall \lambda >0$, recall both $D\rbr{\Ttau||\ptau}$ and $\rbr{\EE_{p\rbr{x}}\sbr{\tau\rbr{x}} - 1}^2$ are non-negative, and the the density ratio $\frac{\mu\rbr{x}}{p\rbr{x}}$ leads to zero for both terms. Then, the density ratio is a solution to $J\rbr{\tau}$. For any other non-negative function $\tau\rbr{x}\ge 0$, if it is the optimal solution to $J\rbr{\tau}$, then, we have 
\begin{eqnarray}\label{eq:opt_condition}
	D\rbr{\Ttau||\ptau}=0 &\Rightarrow&  p\rbr{x'}\tau\rbr{x'} = \Ttau\rbr{x'} = \int\Pb^\pi\rbr{x'|x}\tau\rbr{x}dx,\\ 
	\rbr{\EE_{p\rbr{x}}\sbr{\tau\rbr{x}} - 1}^2 = 0 &\Rightarrow& \EE_{p\rbr{x}}\sbr{\tau\rbr{x}} = 1.
\end{eqnarray}

We denote $\mu\rbr{x} = p\rbr{x}\tau\rbr{x}$, which is clearly a density function. Then, the optimal conditions in \eqref{eq:opt_condition} imply
$$
\mu\rbr{x'} = {\int\Pb^\pi\rbr{x'|x}\mu\rbr{x}dx},
$$
or equivalently, $\mu$ is the stationary distribution of $\Tcal$.  We have thus shown the optimal $\tau\rbr{x} = \frac{\mu\rbr{x}}{p\rbr{x}}$ is the target density ratio.
\end{proof}

%---------------------------------------------------------------------------
\subsection{Convexity of the Objective}\label{appendix:convexity}
%---------------------------------------------------------------------------
\begin{proof}
Since the $\phi$ is convex, we consider the Fenchel dual representation of the $f$-divergence $D_\phi\rbr{\Ttau||\ptau}$, \ie, 
\begin{multline}
D_\phi\rbr{\Ttau||\ptau} = \max_{f\in \Omega\rightarrow \RR} \ell\rbr{\tau, f}\\
\defeq \rbr{1- \gamma}\EE_{\mu_0\pi}\sbr{f\rbr{x}} + \gamma\EE_{\Tcal_p\rbr{x, x'}}\sbr{\tau\rbr{x} f\rbr{x'}}- \EE_{p\rbr{x}}\sbr{\tau\rbr{x}\phi^*\rbr{f\rbr{x}}}.
\end{multline}
% It is obviously $\ell\rbr{\tau, f}$ is convex w.r.t. $\tau$ and concave w.r.t. $f$. Therefore, $D_\phi\rbr{\Ttau||\ptau} = \ell\rbr{\tau, f^*}$ is also convex w.r.t. $\tau$. 
It is obviously $\ell\rbr{\tau, f}$ is convex in $\tau$ for each $f$, then, $D_\phi\rbr{\Ttau||\ptau}$ is convex. The term $\lambda \rbr{\EE_p\rbr{\tau} - 1}^2$ is also convex, which concludes the proof.
\end{proof}

%%%%%%%%%%%%%%%%%%%%%%%%%%%%%%%%%%%%%%%%%%%%%%%%%%%%%%%%%%%%%%%%%%%%%%%%%%%%
\section{Algorithm Details}\label{appendix:alg_details}
%%%%%%%%%%%%%%%%%%%%%%%%%%%%%%%%%%%%%%%%%%%%%%%%%%%%%%%%%%%%%%%%%%%%%%%%%%%%

We provide the unbiased gradient estimator for $\nabla_{w_\tau, u, w_f} J\rbr{\tau, u, f}$ in~\eqnref{eq:chi_saddle_est} below:
\begin{eqnarray}\label{eq:grad_estimator}
% \textstyle
% \hspace{-20mm}
\nabla_{w_\tau}J_{\chi^2}\rbr{\tau, u, f}\hspace{-2mm} &=&\hspace{-2mm} \gamma\EE_{\Tcal_p}\sbr{\nabla_{w_\tau}\tau\rbr{s, a}f\rbr{s', a'}} - \EE_{p}\sbr{\nabla_{w_\tau}\tau\rbr{s, a}\rbr{f\rbr{s, a} + \frac{1}{4}f^2\rbr{s, a}}} \nonumber\\  
&&+\lambda u{\EE_p\sbr{\nabla_{w_\tau}\tau\rbr{s, a}}}, \\
\nabla_{u}J_{\chi^2}\rbr{\tau, u, f} \hspace{-2mm}&=&\hspace{-2mm} \lambda\rbr{\EE_p\sbr{\tau\rbr{s, a}-1}-u},\\
\nabla_{w_f}J_{\chi^2}\rbr{\tau, u, f} \hspace{-2mm}&=&\hspace{-2mm} \rbr{1 - \gamma}\EE_{\mu_0\pi}\sbr{\nabla_{w_f}f\rbr{s, a}} + \gamma\EE_{\Tcal_p}\sbr{\tau\rbr{s, a}\nabla_{w_f}f\rbr{s', a'}} \\
&&- \EE_{p}\sbr{\tau\rbr{s, a}\rbr{1 + \frac{1}{2}f\rbr{s, a}}\nabla_{w_f}f\rbr{s, a}}. \nonumber
\end{eqnarray}

Then, we have the psuedo code which applies SGD for solving~\eqnref{eq:chi_saddle_est}.

\begin{algorithm}[tb]
   \caption{GenDICE (with function approximators)}\label{alg:gendice}
\begin{algorithmic}
   \STATE {\bf Inputs}: Convex function $\phi$ and its Fenchel conjugate $\phi^*$, off-policy data $\Dcal = \{(s^{(i)}, a^{(i)}, r^{(i)}, s^{\prime(i)})\}_{i=1}^N$, initial state $s_0\sim\mu_0$, target policy $\pi$, {distribution corrector}~$\texttt{nn}_{w_\tau}(\cdot,\cdot), \texttt{nn}_{w_f}(\cdot,\cdot)$, constraint scalar $u$,  learning rates $\eta_\tau, \eta_f, \eta_u$, number of iterations $K$,
   batch size $B$.
   \FOR{$t = 1,\dots,K$}
   \STATE Sample batch $\{(s^{(i)}, a^{(i)}, r^{(i)}, s^{\prime(i)})\}_{i=1}^B$ from $\Dset$. 
   \STATE Sample batch $\{s_0^{(i)}\}_{i=1}^B$ from $\mu_0$.
   \STATE Sample actions $a^{\prime(i)}\sim\pi(s^{\prime(i)})$, for $i=1,\dots,B$.
   \STATE Sample actions $a_0^{(i)}\sim\pi(s_0^{(i)})$, for $i=1,\dots,B$.
   \STATE Compute empirical loss $\hat J_{\chi^2}\rbr{\tau, u, f} = \rbr{1 - \gamma}\EE_{\mu_0\pi}\sbr{f\rbr{s, a}} + \gamma\EE_{\Tcal_p}\sbr{\tau\rbr{s, a}f\rbr{s', a'}}$ \\
   $- \EE_{p}\sbr{\tau\rbr{s, a}\rbr{f\rbr{s, a} + \frac{1}{4}f^2\rbr{s, a}}} + \lambda\rbr{\EE_p\sbr{u\tau\rbr{s, a}-u}-\frac{u^2}{2}}$.
   \STATE Update $w_\tau\leftarrow w_\tau - \eta_\tau \nabla_{\theta_\tau}\hat J_{\chi^2}$.
   \STATE Update $w_f\leftarrow w_f + \eta_f \nabla_{\theta_f}\hat J_{\chi^2}$.
   \STATE Update $u\leftarrow u + \eta_u \nabla_{u} \hat J_{\chi^2}$.
   % \Bo{In the proof, I only have guarantees for optimization error from SVRG results. The vanilla SGD result is not known yet. But I think it is possible to derive.}
   \ENDFOR 
   \STATE {\bf Return} $\texttt{nn}_{w_\tau}$.
\end{algorithmic}
\end{algorithm}
%%%%%%%%%%%%%%%%%%%%%%%%%%%%%%%%%%%%%%%%%%%%%%%%%%%%%%%%%%%%%%%%%%%%%%%%%%%%
\section{Proof of~\thmref{thm:total_error}}\label{appendix:proofs}
%%%%%%%%%%%%%%%%%%%%%%%%%%%%%%%%%%%%%%%%%%%%%%%%%%%%%%%%%%%%%%%%%%%%%%%%%%%%

For convenience, we repeat here the notation defined in the main text.  The saddle-point reformulation of the objective function of~\estabb is:
\begin{multline*}
J\rbr{\tau, u, f} \defeq \rbr{1 - \gamma}\EE_{\mu_0\pi}\sbr{f\rbr{x'}} + \gamma\EE_{\Tcal_p\rbr{x, x'}}\sbr{\tau\rbr{x} f\rbr{x'}}\\
- \EE_{p\rbr{x}}\sbr{\tau\rbr{x}\phi^*\rbr{f\rbr{x}}} + \lambda\rbr{\EE_{p\rbr{x}}\sbr{u\tau\rbr{x} - u}  - \frac{1}{2}u^2}.
\end{multline*}
To avoid the numerical infinity in $D_\phi\rbr{\cdot||\cdot}$, we induced the bounded version as
\begin{equation*}
J\rbr{\tau} \defeq \max_{\nbr{f}_\infty\le C, u} \,\, J\rbr{\tau, u, f} = D^C_\phi\rbr{\Ttau || \ptau} + \frac{\lambda}{2} \rbr{\EE_{p\rbr{x}}\sbr{\tau\rbr{x}} - 1}^2,
\end{equation*}
in which $D_\phi^C\rbr{\cdot||\cdot}$ is still a valid divergence, and therefore the optimal solution $\tau^*$ is still the stationary density ratio $\frac{\mu\rbr{x}}{p\rbr{x}}$. We denote the $\Jhat\rbr{\tau,\mu, f}$ as the empirical surrogate of $J\rbr{\tau, \mu, f}$ on samples $\Dcal = \rbr{\rbr{x, x'}_{i=1}^N}\sim \Tpgam\rbr{x, x'}$ with the optimal solution in $\rbr{\Hcal, \Fcal, \RR}$ as $\rbr{\tauhs, \uhs, \fhs}$.  Furthermore, denote 
\begin{align*}
\tau^*_\Hcal &= \argmin_{\tau\in\Hcal} J\rbr{\tau}\,, \\
 \tau^* &= \argmin_{\tau\in S\times A\rightarrow \RR} J\rbr{\tau}
\end{align*}
with optimal $\rbr{f^*, u^*}$, and
\begin{align*}
L\rbr{\tau} &= \max_{f\in \Fcal, u\in \RR}J\rbr{\tau, u, f}\,, \\
\Lhat\rbr{\tau} &= \max_{f\in\Fcal, u\in\RR} \Jhat\rbr{\tau,u, f}\,.
\end{align*}
We apply some optimization algorithm for $\Jhat\rbr{\tau, u, f}$ over space $\rbr{\Hcal, \Fcal, \RR}$, leading to the output $\rbr{\hat\tau, \uhat, \fhat}$. 
% We consider the norm $\nbr{f}_2 \defeq \inner{f}{f}_{p} = \EE_{p}\sbr{f^2\rbr{x}}$. 

Under~\asmpref{asmp:ref_dist}, we need only consider $\nbr{\tau}_\infty\le C$, then, the corresponding dual $u = \EE_p\rbr{\tau} - 1\Rightarrow u\in U\defeq\cbr{\abr{u}\le \rbr{C+1}}$. We choose the $\phi^*\rbr{\cdot}$ is a $\kappa$-Lipschitz continuous, then, the $J\rbr{\tau, u, f}$ is a $C_{\Pb^\pi, \kappa, \lambda} = \max\cbr{\rbr{{\gamma}\nbr{\Pb^\pi}_{p, \infty} +\rbr{1 - \gamma} \nbr{\frac{\mu_0\pi}{p}}_{p, \infty} + \kappa}C, \rbr{C+1}\rbr{\lambda + \frac{1}{2}}}$-Lipschitz continuous function w.r.t. $\rbr{f, u}$ with the norm $\nbr{\rbr{f, u}}_{p, 1}\defeq \int \abr{f\rbr{x}}p\rbr{x}dx + \abr{u}$, and $C_{\phi, C, \lambda}\defeq \rbr{C + \lambda\rbr{C+1} + \max_{t\in \cbr{-C, C}}\rbr{-\phi\rbr{t}}}$-Lipschitz continuous function w.r.t. $\tau$ with the norm $\nbr{\tau}_{p, 1} \defeq \int \abr{\tau\rbr{x}}p\rbr{x}dx$. 

We consider the error between $\hat\tau$ and $\tau^*$ using standard arguments~\citep{ShaBen14,Bach14}, \ie, 
\begin{equation*}
d\rbr{\hat\tau, \tau^*}\defeq J\rbr{\hat\tau} - J\rbr{\tau^*}.
% = \underbrace{J\rbr{\hat\tau} - J\rbr{\tauhs}}_{\eps_1} + \underbrace{J\rbr{\tauhs} - J\rbr{\tau^*}}_{\eps_2}. 
\end{equation*}
The discrepancy $d\rbr{\tau, \tau^*}\ge 0$ and $d\rbr{\tau, \tau^*} = 0$ if and only if $\ptau$ is stationary distribution of $\Tcal$ in the weak sense of $D_\phi\rbr{\cdot||\cdot}$.

\paragraph{Remark:} In some special cases, the suboptimality also implies the distance between $\hat\tau$ and $\tau^*$.  
% \lihong{A bit lost here: (1) what is the intended message? (2) $\Qcal_{\setminus\tau^*}$ undefined?}
% \Bo{ (1), I would like to justify that the $d\rbr{\tau, \tau^*}$ can be used to evalute the difference directly, instead of just the suboptimality. In other words, I would like to have some strong informal claim in main text about the estimated expectation with this argument. (2) $\Qcal_{\setminus\tau^*}$ means the all other eigenfunctions except $\tau^*$. }
Specifically,for $\gamma=1$, if the transition operator $\Pb^\pi$ can be represented as $\Pb^\pi = \Qcal\Lambda \Qcal^{-1}$ where $\Qcal$ denotes the (countable) eigenfunctions and $\Lambda$ denotes the diagonal matrix with eigenvalues, the largest of which is $1$.  We consider $\phi\rbr{\cdot}$ as identity and $f\in \Fcal\defeq\cbr{\spn\rbr{\Qcal}, \nbr{f}_{p, 2}\le 1}$, then the $d\rbr{\tau, \tau^*}$ will bounded from below by a metric between $\tau$ and $\tau^*$. Particularly, we have
\begin{eqnarray*}
D_\phi\rbr{\Ttau||\ptau}= \max_{f\in \Fcal}\,\, \EE_{\Tcal_p\rbr{x, x'}}\sbr{\tau\rbr{x} f\rbr{x'}}- \EE_{p\rbr{x}}\sbr{\tau\rbr{x}{f\rbr{x}}} = \nbr{\tau - \Pb^\pi\circ\tau}_{p, 2}.
\end{eqnarray*}
Rewrite $\tau = \alpha\tau^* + \zeta$, where $\zeta\in \spn\rbr{\Qcal_{\setminus\tau^*}}$, then 
\begin{eqnarray*}
D_\phi\rbr{\Ttau||\ptau} = \nbr{\alpha\tau^* - \alpha\Pb^\pi\circ\tau^* + \zeta - \Pb^\pi\circ\zeta}_{p, 2} = \nbr{\zeta - \Pb^\pi\circ\zeta}_{p, 2}.
\end{eqnarray*}
Recall the optimality of $\tau^*$, \ie, $D_\phi\rbr{\Ttau^*||\ptau^*} = 0$, we have 
\begin{equation*}
d\rbr{\tau, \tau^*} =J\rbr{\tau} \ge \nbr{\zeta - \Pb^\pi\circ\zeta}_{p, 2} \defeq \nbr{\rbr{\tau-\tau^*}}_{p, 2, \rbr{\Pb^\pi - I}}.
\end{equation*}

%---------------------------------------------------------------------------
\subsection{Error Decomposition}

We start with the following error decomposition:
\begin{equation*}
d\rbr{\hat\tau, \tau^*}\defeq J\rbr{\hat\tau} - J\rbr{\tau^*}
= \underbrace{J\rbr{\hat\tau} - J\rbr{\tauhs}}_{\eps_1} + \underbrace{J\rbr{\tauhs} - J\rbr{\tau^*}}_{\eps_2}. 
\end{equation*}

\begin{itemize}[leftmargin=*]
	\item For $\eps_1$, we have
\begin{eqnarray*}
\eps_1 &=& J\rbr{\hat\tau} - L\rbr{\hat\tau} + L\rbr{\hat\tau} - L\rbr{\tauhs} + L\rbr{\tauhs} - J\rbr{\tauhs}.
\end{eqnarray*}
We consider the terms one-by-one. By definition, we have
\begin{eqnarray}
J\rbr{\hat\tau} - L\rbr{\hat\tau} &=& \max_{f, \mu} J\rbr{\hat\tau, u, f} - \max_{f\in \Fcal, \mu}J\rbr{\hat\tau, u, f} \nonumber \\ 
% &=&\max_{f_1, \mu_1}\min_{f_2\in \Fcal, \mu_2} J\rbr{\hat\tau, f_1, \mu_1} - J\rbr{\hat\tau, f_2, \mu_2}
&\le&C_{\Pb^\pi, \kappa, \lambda}\underbrace{\sup_{f_1, u_1\in U}\inf_{f_2\in\Fcal, u_2\in U}\nbr{\rbr{f_1, u_1} - \rbr{f_2, u_2}}_{p, 1}}_{\eps_{approx}\rbr{\Fcal}}, \label{eq:eps_fapprox}
\end{eqnarray}
which is induced by introducing $\Fcal$ for dual approximation.

For the third term $L\rbr{\tauhs} - J\rbr{\tauhs}$, we have 
\begin{eqnarray*}
L\rbr{\tauhs} - J\rbr{\tauhs} = \max_{f\in \Fcal, u\in U} J\rbr{\tauhs, u, f} - \max_{f, u\in U} J\rbr{\tauhs, u, f} \le 0. 
\end{eqnarray*}
For the term $L\rbr{\hat\tau} - L\rbr{\tauhs}$, 
\begin{eqnarray}
 L\rbr{\hat\tau} - L\rbr{\tauhs} &=&  L\rbr{\hat\tau} - \Lhat\rbr{\hat\tau} + \underbrace{\Lhat\rbr{\hat\tau} - \Lhat\rbr{\tauhs}}_{\hat\eps_{opt}} + \Lhat\rbr{\tauhs} - L\rbr{\tauhs} \nonumber \\
 &\le& 2\sup_{\tau\in \Hcal}\abr{L\rbr{\tau} - \Lhat\rbr{\tau}} + \hat\eps_{opt} \nonumber \\
 &\le& 2\sup_{\tau\in \Hcal}\abr{\max_{f\in\Fcal, u\in U} J\rbr{\tau, u, f} - \max_{f\in \Fcal, u\in U}\Jhat\rbr{\tau, u, f}} + \hat\eps_{opt} \nonumber \\
 &\le& 2\sup_{\tau\in\Hcal}\sup_{f\in \Fcal, u\in U}\abr{J\rbr{\tau, u, f} - \Jhat\rbr{\tau, u, f}} + \hat\eps_{opt} \nonumber \\
 &=& 2\cdot \eps_{est} + \hat\eps_{opt}, \label{eq:eps_est}
\end{eqnarray}
where we define $\eps_{est}\defeq \sup_{\tau\in\Hcal, f\in \Fcal, u\in U}\abr{J\rbr{\tau, u, f} - \Jhat\rbr{\tau, u, f}}$. 

Therefore, we can now bound $\eps_1$ as
\begin{equation*}
\eps_1\le C_{\Tcal, \kappa, \lambda}\eps_{approx}\rbr{\Fcal} + 2\eps_{est} + \hat\eps_{opt}. 
\end{equation*}

\item For $\eps_2$, we have
\begin{eqnarray*}
\eps_2 &=& J\rbr{\tauhs} - J\rbr{\tau_\Hcal^*} + J\rbr{\tau_\Hcal^*} - J\rbr{\tau^*}\\
&=& J\rbr{\tauhs} - L\rbr{\tauhs} + L\rbr{\tauhs} - L\rbr{\tau_\Hcal^*} + L\rbr{\tau_\Hcal^*} - J\rbr{\tau_\Hcal^*} + J\rbr{\tau_\Hcal^*} - J\rbr{\tau^*}.
\end{eqnarray*}

We consider the terms from right to left. For the term $J\rbr{\tau_\Hcal^*} - J\rbr{\tau^*}$, we have
\begin{eqnarray*}
J\rbr{\tauhs} - J\rbr{\tau^*} =& J\rbr{\tauhs, \uhs, \fhs} - J\rbr{\tau^*, \uhs, \fhs} + \underbrace{J\rbr{\tau^*, \uhs, \fhs} - J\rbr{\tau^*, u^*, f^*}}_{\le 0}\\
=& J\rbr{\tauhs, \uhs, \fhs} - J\rbr{\tau^*, \uhs, \fhs} \le C_{\phi, C, \lambda}\underbrace{\sup_{\tau_1}\inf_{\tau_2\in \Hcal}\nbr{\tau_1 - \tau_2}_{p, 1}}_{\eps_{approx}\rbr{\Hcal}},
\end{eqnarray*}
which is induced by restricting the function space to $\Hcal$. The second term is nonpositive, due to the optimality of $\rbr{u^*, f^*}$. The final inequality comes from the fact that $J\rbr{\tau, u, f}$ is $C_{\phi, C, \lambda}$-Lipschitz w.r.t. $\tau$. 

For the term $L\rbr{\tauhs} - J\rbr{\tau_\Hcal^*}$, by definition
\begin{equation*}
L\rbr{\tau_\Hcal^*} - J\rbr{\tau_\Hcal^*} = \max_{f\in \Fcal, u\in U} J\rbr{\tau_{\Hcal}^*, f, u} - \max_{f, u\in U} J\rbr{\tau_{\Hcal}^*, f, u}\le 0.
\end{equation*}

For the term $ L\rbr{\tauhs} - L\rbr{\tau_\Hcal^*}$, we have
\begin{eqnarray*}
L\rbr{\tauhs} - L\rbr{\tau_\Hcal^*} &=&  L\rbr{\tauhs} - \Lhat\rbr{\tauhs} + \underbrace{\Lhat\rbr{\tauhs} - \Lhat\rbr{\tau_\Hcal^*}}_{\le 0} + \Lhat\rbr{\tau_\Hcal^*} - L\rbr{\tau_\Hcal^*}\\
&=& 2\sup_{\tau\in\Hcal}\abr{L\rbr{\tau} - \Lhat\rbr{\tau}}\\
&=& 2\sup_{\tau\in\Hcal, f\in \Fcal, u\in U}\abr{J\rbr{\tau, u, f} - \Jhat\rbr{\tau, u, f}}\\
&=& 2\cdot\eps_{est}. 
\end{eqnarray*}
where the second term is nonpositive, thanks to the optimality of $\tauhs$. 

Finally, for the term $J\rbr{\tauhs} - J\rbr{\tau_\Hcal^*}$, using the same argument in~\eqref{eq:eps_fapprox}, we have 
\begin{equation*}
J\rbr{\tauhs} - J\rbr{\tau_\Hcal^*}\le C_{\Pb^\pi,\kappa, \lambda}\eps_{approx}\rbr{\Fcal}. 
\end{equation*}
Therefore, we can bound $\eps_2$ by
\begin{equation*}
\eps_2 \le C_{\phi, C, \lambda}\eps_{approx}\rbr{\Hcal} + C_{\Pb^\pi, \kappa, \lambda}\rbr{\Fcal} + 2\eps_{est}. 
\end{equation*}

\end{itemize}

In sum, we have 
\begin{equation*}
d\rbr{\hat\tau, \tau^*} \le 4\eps_{est} + \hat\eps_{opt} + 2C_{\Pb^\pi, \kappa,\lambda}\eps_{approx}\rbr{\Fcal} + C_{\phi, C, \lambda}\eps_{approx}\rbr{\Hcal}.
\end{equation*}
In the following sections, we will bound the $\eps_{est}$ and $\hat\eps_{opt}$. 

%---------------------------------------------------------------------------
\subsection{Statistical Error}
%---------------------------------------------------------------------------

In this section, we analyze the statistical error 
$$
\eps_{est}\defeq \sup_{\tau\in\Hcal, f\in \Fcal, u\in U}\abr{J\rbr{\tau, u, f} - \Jhat\rbr{\tau, u, f}}.
$$
We mainly focus on the batch RL setting with \iid~samples $\Dcal = \sbr{\rbr{x_i, x_i'}}_{i=1}^N\sim\Tcal_p\rbr{x, x'}$, which has been studied by previous authors~\citep[e.g.,][]{SutSzeGerBow12,nachum2019dualdice}. However, as discussed in the literature~\citep{AntSzeMun08b,LazGhaMun12,DaiShaLiXiaHeetal17,nachum2019dualdice}, using the blocking technique of~\citet{Yu94}, the statistical error provided here can be generalized to $\beta$-mixing samples in a single sample path. We omit this generalization for the sake of expositional simplicity. 

To bound the $\eps_{est}$, we follow similar arguments by~\citet{DaiShaLiXiaHeetal17,nachum2019dualdice} via the covering number. For completeness, the definition is given below.

The Pollard's tail inequality bounds the maximum deviation via the covering number of a function class:
\begin{lemma}[\citet{Pollard12}]\label{lemma:beta_tail}
Let $\gG$ be a permissible class of $\gZ\rightarrow[-M, M]$ functions and $\cbr{Z_i}_{i=1}^N$ are \iid~samples from some distribution.  Then, for any given $\epsilon>0$,
\begin{eqnarray*}
\sP\rbr{\sup_{g\in\gG}\abr{\frac{1}{N}\sum_{i=1}^N g(Z_i) - \E\sbr{g(Z)}} > \epsilon}\le 8\E\sbr{\gN_1\rbr{\frac{\epsilon}{8}, \gG, \cbr{Z_i}_{i=1}^N}}\exp\rbr{\frac{-N\epsilon^2}{512M^2}}.
\end{eqnarray*}
\end{lemma}
The covering number can then be bounded in terms of the function class's pseudo-dimension:
\begin{lemma}[\citet{Haussler95}, Corollary~3]\label{lemma:cover_pseudo}
For any set $\Xcal$, any points $x^{1:N}\in\Xcal^N$, any class $\Fcal$ of functions on $\Xcal$ taking values in $[0, M]$ with pseudo-dimension $D_{\Fcal}<\infty$, and any $\epsilon>0$, 
$$
\Ncal_1\rbr{\epsilon, \Fcal, x^{1:N}}\le e\rbr{D_\Fcal + 1}\rbr{\frac{2eM}{\epsilon}}^{D_\Fcal}.
$$
\end{lemma}
The statistical error $\eps_{est}$ can be bounded using these lemmas.
\begin{lemma}[Stochastic error]\label{lemma:stat_error}
Under the~\asmpref{asmp:ref_dist}, if $\phi^*$ is $\kappa$-Lipschitz continuous and the psuedo-dimension of $\Hcal$ and $\Fcal$ are finite, with probability at least $1-\delta$, we have
$$
\eps_{est} = \Ocal\rbr{\sqrt{\frac{\log N + \log\frac{1}{\delta}}{N}}}.
$$
\end{lemma}

\begin{proof}
The proof works by verifying the conditions in~\lemref{lemma:beta_tail} and computing the covering number. 

Denote the $h_{\tau, u, f}\rbr{x,x'} =\rbr{1 - \gamma} f\rbr{x'} + \gamma \tau\rbr{x}f\rbr{x'} - \tau\rbr{x}\phi^*\rbr{f\rbr{x}} + \lambda u\tau\rbr{x} - \lambda u - \lambda\frac{1}{2}u^2$, we will apply~\lemref{lemma:beta_tail} with $\Zcal = \Omega\times\Omega$, $Z_i = \rbr{x_i, x_i'}$, and $\Gcal = h_{\Hcal\times\Fcal\times U}$. 

We check the boundedness of $h_{\zeta, u, f}\rbr{x, x'}$. Based on~\asmpref{asmp:ref_dist}, we only consider the $\tau\in \Hcal$ and $u\in U$ bounded by $C$ and $C+1$. We also rectify the $\nbr{f}_\infty \le C$. Then, we can bound the $\nbr{h}_\infty$:
\begin{eqnarray*}
\nbr{h_{\tau, u, f}}_\infty &\le& \rbr{1 + \nbr{\tau}_\infty}\nbr{f}_\infty + \nbr{\tau}_\infty\rbr{\max_{t\in [-C, C]}\,\, -\phi^*\rbr{t}} + \lambda C\rbr{\nbr{\tau}_\infty + 1} + \lambda C^2\\
&\le& \rbr{C+1}^2 + C\cdot C_\phi +  \lambda C\rbr{2C+1} =: M,
\end{eqnarray*}
where $C_\phi = \max_{t\in [-C, C]}\,\, -\phi^*\rbr{t}$. Thus, by~\lemref{lemma:beta_tail}, we have
\begin{eqnarray}
\lefteqn{\sP\rbr{\sup_{\tau\in\Hcal, f\in \Fcal, u\in U}\abr{\Jhat\rbr{\tau, u, f} - J\rbr{\tau, u, f}}}} \nonumber \\
&=& \sP\rbr{\sup_{\tau\in\Hcal, f\in \Fcal, u\in U}\abr{\frac{1}{n}\sum_{i=1}^N h_{\zeta, u, f}\rbr{Z_i} - \EE\sbr{h_{\zeta, u, f}}}} \nonumber \\
&\le& \E\sbr{\gN_1\rbr{\frac{\epsilon}{8}, \gG, \cbr{Z_i}_{i=1}^N}}\exp\rbr{\frac{-N\epsilon^2}{512M^2}}. \label{eq:intermediate}
\end{eqnarray}
Next, we check the covering number of $\Gcal$. Firstly, we bound the distance in $\Gcal$, 
\begin{eqnarray*}
&&\frac{1}{N}\sum_{i=1}^N\abr{h_{\tau_1, u_1, f_1}\rbr{Z_i} - h_{\tau_2, u_2, f_2}\rbr{Z_i}}\\
&\le&\frac{C + C_\phi + \lambda(C+1)}{N}\sum_{i=1}^N \abr{\tau_1\rbr{x_i} -\tau_2\rbr{x_i}} + \frac{1 + \gamma C}{N}\sum_{i=1}^N \abr{f_1\rbr{x'_i} - f_2\rbr{x'_i}} \\
&&+ \frac{\kappa C}{N}\sum_{i=1}^N \abr{f_1\rbr{x_i} - f_2\rbr{x_i}} + \lambda\rbr{2C +1}\abr{u_1 - u_2}, 
\end{eqnarray*}
which leads to
% \lihong{How does the inequality below work?  Maybe expand a bit or explain?} 
% \Bo{This comes from above inequality about the distance in $\Gcal$. The covering number of the direct product of two function space is bounded by the product of the covering number of the two spaces (with the diameter of the covering balls summed up.) We used this for SBEED and DualDICE.....I think it is correct. }
\begin{eqnarray*}
&&\Ncal_1\rbr{\rbr{C_\phi + \rbr{3\lambda + 2 + \gamma + \kappa}\rbr{C+1}}\eps', \Gcal, \cbr{Z_i}_{i=1}^N}\\
&\le& \Ncal_1\rbr{\eps', \Hcal, \rbr{x_i}_{i=1}^N}\Ncal_1\rbr{\eps', \Fcal, \rbr{x'_i}_{i=1}^N}\Ncal_1\rbr{\eps', \Fcal, \rbr{x_i}_{i=1}^N}\Ncal_1\rbr{\eps', U}.
\end{eqnarray*}
For the set $U = [-C-1, C+1]$, we have,
\begin{equation*}
\Ncal_1\rbr{\eps', U}\le \frac{2C+2}{\eps'}.
\end{equation*}
Denote the pseudo-dimension of $\Hcal$ and $\Fcal$ as $D_\Hcal$ and $D_\Fcal$, respectively, we have
\begin{align*}
% \textstyle
& \Ncal_1\rbr{\rbr{C_\phi + \rbr{3+2\lambda + \kappa}\rbr{C+1}}\eps', \Gcal, \cbr{Z_i}_{i=1}^N} \\
\le&  e^3\rbr{D_\Hcal+1}\rbr{D_\Fcal+1}^2\rbr{\frac{2C+2}{\eps'}}\rbr{\frac{4eC}{\eps'}}^{D_\Hcal + 2D_\Fcal},
\end{align*}
which implies
\begin{eqnarray*}
&&\Ncal_1\rbr{\frac{\eps}{8}, \Gcal, \cbr{Z_i}_{i=1}^N}\\
&\le& \frac{C+1}{2C}e^2\rbr{D_\Hcal+1}\rbr{D_\Fcal+1}^2\rbr{\frac{32\rbr{C_\phi + \rbr{3\lambda + 2 + \gamma + \kappa}\rbr{C+1}}eC}{\eps}}^{D_\Hcal+D_\Fcal + 1}\\
&=& C_1\rbr{\frac{1}{\eps}}^{D_1},
\end{eqnarray*}
where $D_1 = D_\Hcal + D_\Fcal + 1$ and
$$
C_1 = \frac{C+1}{2C}e^2\rbr{D_\Hcal+1}\rbr{D_\Fcal+1}^2\rbr{32\rbr{C_\phi + \rbr{3\lambda + 2 + \gamma + \kappa}\rbr{C+1}}eC}\,.
$$

Combine this result with~\eqref{eq:intermediate}, we obtain the bound for the statistical error:
\begin{equation}
\sP\rbr{\sup_{\tau\in\Hcal, f\in \Fcal, u\in U}\abr{\Jhat\rbr{\tau, u, f} - J\rbr{\tau, u, f}}} \le 8C_1\rbr{\frac{1}{\eps}}^{D_1}\exp\rbr{\frac{-N\epsilon^2}{512M^2}}.
\end{equation}
Setting $\eps = \sqrt{\frac{C_2\rbr{\log N + \log \frac{1}{\delta}}}{N}}$ with $C_2 = \max\rbr{\rbr{8C_1}^{\frac{2}{D_1}}, 512MD_1, 512M, 1}$, we have
\begin{equation*}
 8C_1\rbr{\frac{1}{\eps}}^{D_1}\exp\rbr{\frac{-N\epsilon^2}{512M^2}}\le \delta.
\end{equation*}
\end{proof}

%---------------------------------------------------------------------------
\subsection{Optimization Error}\label{appendix:opt_error}
%---------------------------------------------------------------------------
In this section, we investigate the optimization error
$$
\hat\eps_{opt}\defeq\Lhat\rbr{\hat\tau} - \Lhat\rbr{\tauhs}.
$$
Notice our estimator $\min_{\tau\in\Hcal}\max_{f\in \Fcal, u\in U}\Jhat\rbr{\tau, u, f}$ is compatible with different parametrizations for $\rbr{\Hcal, \Fcal}$ and different optimization algorithms, the optimization error will be different. For the general neural network for $\rbr{\tau, f}$, although there are several progress recently~\citep{LinLiuRafYan18,JinNetJor19,LinJinJor19} about the convergence to a stationary point or local minimum, it remains a largely open problem to quantify the optimization error, which is out of the scope of this paper.  Here, we mainly discuss the convergence rate with tabular, linear and kernel parametrization for $\rbr{\tau, f}$.

% Without loss of generality, \lihong{Not really ``WLOG''?}
% Particularly, we consider the linear parametrization particularly, \ie, $\tau\rbr{x} =  \sigma\rbr{w_\tau^\top \psi\rbr{x}}$, $f\rbr{x} = w_f^\top \psi\rbr{x}$, and $\sigma\rbr{\cdot}:\RR\rightarrow\RR_+$ is convex. There are many choices of the $\sigma\rbr{\cdot}$, \eg, $\exp\rbr{\cdot}$, $\log\rbr{1 + \exp\rbr{\cdot}}$ and $\rbr{\cdot}^2$. Obviously, even with such nonlinear mapping, the $\Jhat\rbr{\tau, u, f}$ is still convex-concave w.r.t $\rbr{w_\tau, w_f, u}$ by the convex composition rule. 

Particularly, we consider the linear parametrization particularly, \ie, $\tau\rbr{x} =  {w_\tau^\top \psi_\tau\rbr{x}}$ with $\cbr{w_\tau, \psi_\tau\rbr{x}\ge 0} $ and $f\rbr{x} = w_f^\top \psi_f\rbr{x}$.  With such parametrization, the $\Jhat\rbr{\tau, u, f}$ is still convex-concave w.r.t $\rbr{w_\tau, w_f, u}$.

We can bound the $\hat\eps_{opt}$ by the primal-dual gap $\eps_{gap}$: 
\begin{eqnarray*}
\hat\eps_{opt} &=& \Lhat\rbr{\hat\tau} - \Lhat\rbr{\tauhs}\\
&\le& \max_{f\in\Fcal, u\in U}\Jhat\rbr{\hat\tau, u, f} - \Jhat\rbr{\tauhs, \uhs, \fhs} +  \Jhat\rbr{\tauhs, \uhs, \fhs} - \min_{\tau\in \Hcal}\Jhat\rbr{\tau, \uhat, \fhat}\\
&=& \underbrace{\max_{f\in\Fcal, u\in U}\Jhat\rbr{\hat\tau, u, f} - \min_{\tau\in \Hcal}\Jhat\rbr{\tau, \uhat, \fhat}}_{\eps_{gap}}.
\end{eqnarray*}
With vanilla SGD, we have $\eps_{gap} = \Ocal\rbr{\frac{1}{\sqrt{T}}}$, where $T$ is the optimization steps~\citep{NemJudLanSha09}. 
%rate in terms of the $\eps_{gap}$~\citep{NemJudLanSha09},
Therefore, $\eps_{opt} = \EE\sbr{\hat\eps_{opt}} = \Ocal\rbr{\frac{1}{\sqrt{T}}}$, where the $\EE\sbr{\cdot}$ is taken w.r.t. randomness in SGD. 

%---------------------------------------------------------------------------
\subsection{Complete Error Analysis}\label{appendix:full_error}
%---------------------------------------------------------------------------

We are now ready to state the main theorm in a precise way:
\paragraph{\thmref{thm:total_error}}
\textit{
Under Assumptions~\ref{asmp:ref_dist} and~\ref{asmp:stat_exist} , the stationary distribution $\mu$ exists, \ie, $\max_{f\in\Fcal^*}\EE_{\Tcal\circ\mu}\sbr{f} - \EE_{\mu}\sbr{\phi^*\rbr{f}} = 0$. If the $\phi^*\rbr{\cdot}$ is $\kappa$-Lipschitz continuous, $\nbr{f}_\infty\le C<\infty,\,\,\forall f\in \Fcal^*$, and the psuedo-dimension of $\Hcal$ and $\Fcal$ are finite, the error between the~\estname estimate to $\tau^*\rbr{x} = \frac{u\rbr{x}}{p\rbr{x}}$ is bounded by
$$
\EE\sbr{J\rbr{\hat\tau} - J\rbr{\tau^*}} =\widetilde\Ocal\rbr{\eps_{approx}\rbr{\Fcal, \Hcal} + \sqrt{\frac{1}{N}} + \eps_{opt}},
$$
where $\EE\sbr{\cdot}$ is w.r.t. the randomness in sample $\Dcal$ and in the optimization algorithms. $\eps_{opt}$ is the optimization error, and $\eps_{approx}\rbr{\Fcal, \Hcal}$ is the approximation induced by $\rbr{\Fcal, \Hcal}$ for parametrization of $\rbr{\tau, f}$.
% \lihong{``stochastic algorithms'' to ``optimization algorithm''?}\Bo{sure!}
}
\begin{proof}
We have the total error as 
\begin{equation}\label{eq:bound}
\EE\sbr{J\rbr{\hat\tau} - J\rbr{\tau^*}}  \le 4\EE\sbr{\eps_{est}} + \EE\sbr{\eps_{opt}} + \eps_{approx}\rbr{\Fcal, \Hcal}, 
\end{equation}
where $\eps_{approx}\defeq 2C_{\Tcal, \kappa,\lambda}\eps_{approx}\rbr{\Fcal} + C_{\phi, C, \lambda}\eps_{approx}\rbr{\Hcal}$. For ${\eps_{opt}}$, we can apply the results for SGD in~\appref{appendix:opt_error}. 

We can bound the $\EE\sbr{\eps_{est}}$ by \lemref{lemma:stat_error}. Specifically, we have
\begin{equation*}
\EE\sbr{\eps_{est}} = \rbr{1 - \delta}\sqrt{\frac{C_2\rbr{\log N + \log \frac{1}{\delta}}}{N}} + \delta M =\Ocal\rbr{\sqrt{\frac{\log N}{N}}},
\end{equation*}
by setting $\delta = \frac{1}{\sqrt{N}}$. 

Plug all these bounds into~\eqref{eq:bound}, we achieve the conclusion.

\end{proof}

%!TEX root = dual_ratio.tex

% \clearpage

%%%%%%%%%%%%%%%%%%%%%%%%%%%%%%%%%%%%%%%%%%%%%%%%%%%%%%%%%%%%%%%%%%%%%%%%%%%%
\section{Experimental Settings}~\label{appendix:exp_settings}
%%%%%%%%%%%%%%%%%%%%%%%%%%%%%%%%%%%%%%%%%%%%%%%%%%%%%%%%%%%%%%%%%%%%%%%%%%%%

\subsection{Tabular Case}
For the Taxi domain, we follow the same protocol as used in~\citet{liu2018breaking}. The behavior and target policies are also taken from~\citet{liu2018breaking} (referred in their work as the behavior policy for $\alpha=0$).
We use a similar taxi domain, where a grid size of $5 \times 5$ yields $2000$ states in total ($25\times16\times5$, corresponding to $25$ taxi locations, 16 passenger appearance status and $5$ taxi status). We set our target policy as the final policy $\pi_*$ after running Q-learning~\citep{SutBar98} for $1000$ iterations, and set another policy $\pi_+$ after $950$ iterations as our base policy. The behavior policy is a mixture policy controlled by $\alpha$ as
$\pi = (1-\alpha)\pi_* + \alpha \pi_+$, \ie, the larger $\alpha$ is, the behavior policy is more close to the target policy. In this setting, we solve for the optimal stationary ratio $\tau$ exactly using matrix operations.  Since~\citet{liu2018breaking} perform a similar exact solve for $|S|$ variables $\mu(s)$, for better comparison we also perform our exact solve with respect to $|S|$ variables $\tau(s)$.  Specifically, the final objective of importance sampling will require knowledge of the importance weights $\mu(a|s)/p(a|s)$.
~~~~~~~~~~
\begin{wraptable}{r}{0.55\textwidth}
	\small
	\caption{Statistics of different graphs.}
	\begin{tabular}{lcc}
 	% \begin{adjustbox}{scale=0.9,tabular= lcc,center}
		\toprule[1.2pt]
		\textbf{Dataset} & \textbf{Number of Nodes} & \textbf{Number of Edges}\\
		\midrule
		BA (Small) &  $100$ & $400$\\
		BA (Large) & $500$ & $2000$\\
		Cora & $2708$ & $5429$\\
		Citeseer & $3327$ & $4731$\\
		\bottomrule[1.2pt]
		\vspace{-6mm}
	% \end{adjustbox}
	\end{tabular}
	% \vspace{-4mm}
	\label{tab:oprdata}
\end{wraptable} 

\vspace{-2mm}
For offline PageRank, the graph statistics are illustrated in Table \ref{tab:oprdata}, and the degree statistics and graph visualization are shown in Figure \ref{fig:offpolicy-graphs}. For the Barabasi–Albert (BA) Graph, it begins with an initial connected network of $m_0$ nodes in the network. Each new node is connected to $m\leq m_{0}$ existing nodes with a probability that is proportional to the number of links that the existing nodes already have. 
% We are targeting on matching $p$ and $\frac{\Ttau}{\tau}$ by minimizing $D\rbr{p||\frac{\Ttau}{\tau}}$ via its Fenchel dual form \Roy{Using $D\rbr{p{\tau}||{\Ttau}}$ now.}. 
Intuitively, heavily linked nodes (`hubs') tend to quickly accumulate even more links, while nodes with only a few links are unlikely to be chosen as the destination for a new link. The new nodes have a `preference' to attach themselves to the already heavily linked nodes. For two real-world graphs, it is built upon the real-world citation networks. In our experiments, the weights of the BA graph is randomly drawn from a standard Gaussian distribution with normalization to ensure the property of the transition matrix. The offline data is collected by a random walker on the graph, which consists the initial state and next state in a single trajectory. In experiments, we vary the number of off-policy samples to validate the effectiveness of \estname with limited offline samples provided.
\begin{figure}[ht] \centering
	\begin{tabular}{cccc}
		\hspace{-3mm}
		\includegraphics[width=0.25\linewidth]{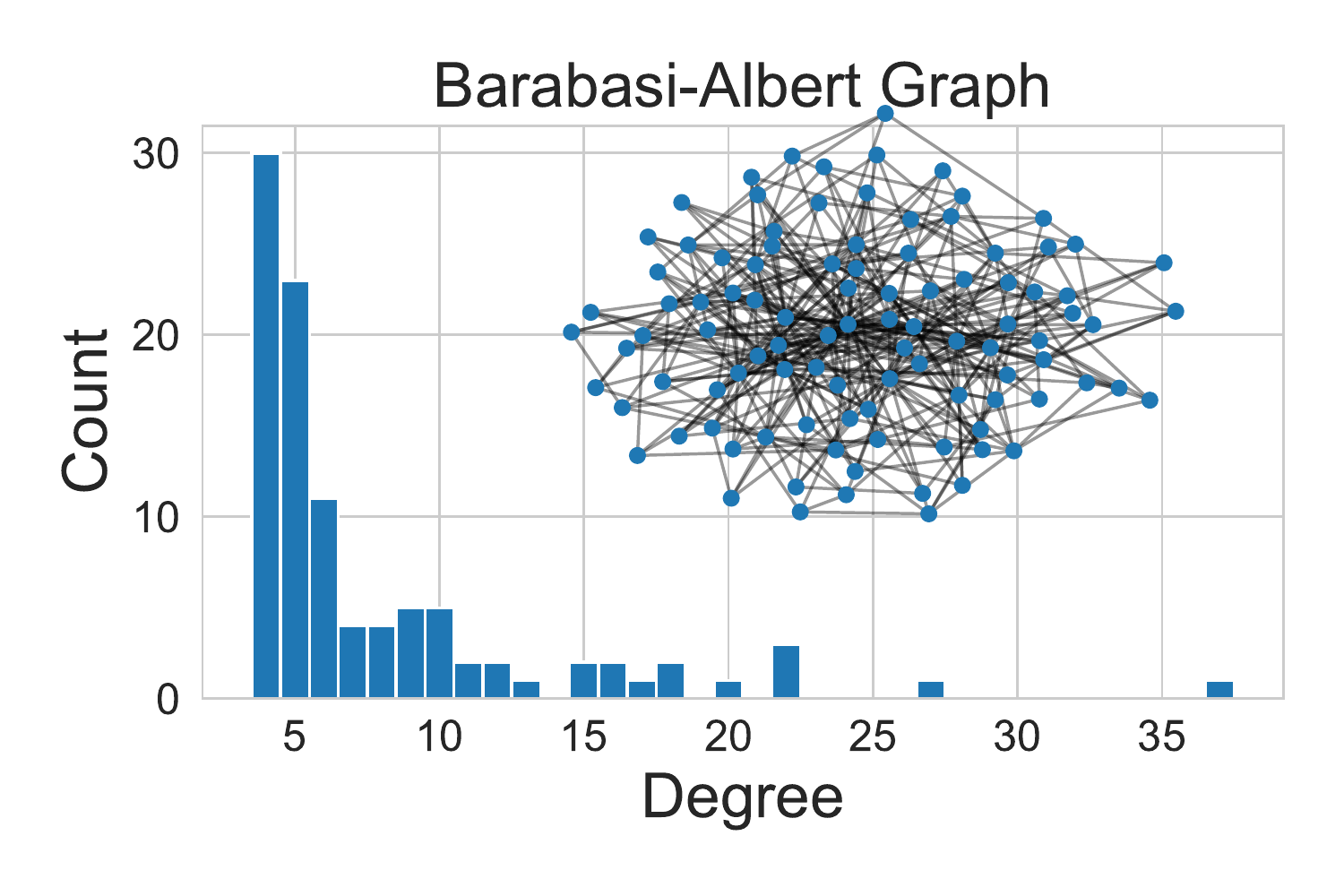}  &
		\hspace{-5mm}
		\includegraphics[width=0.25\linewidth]{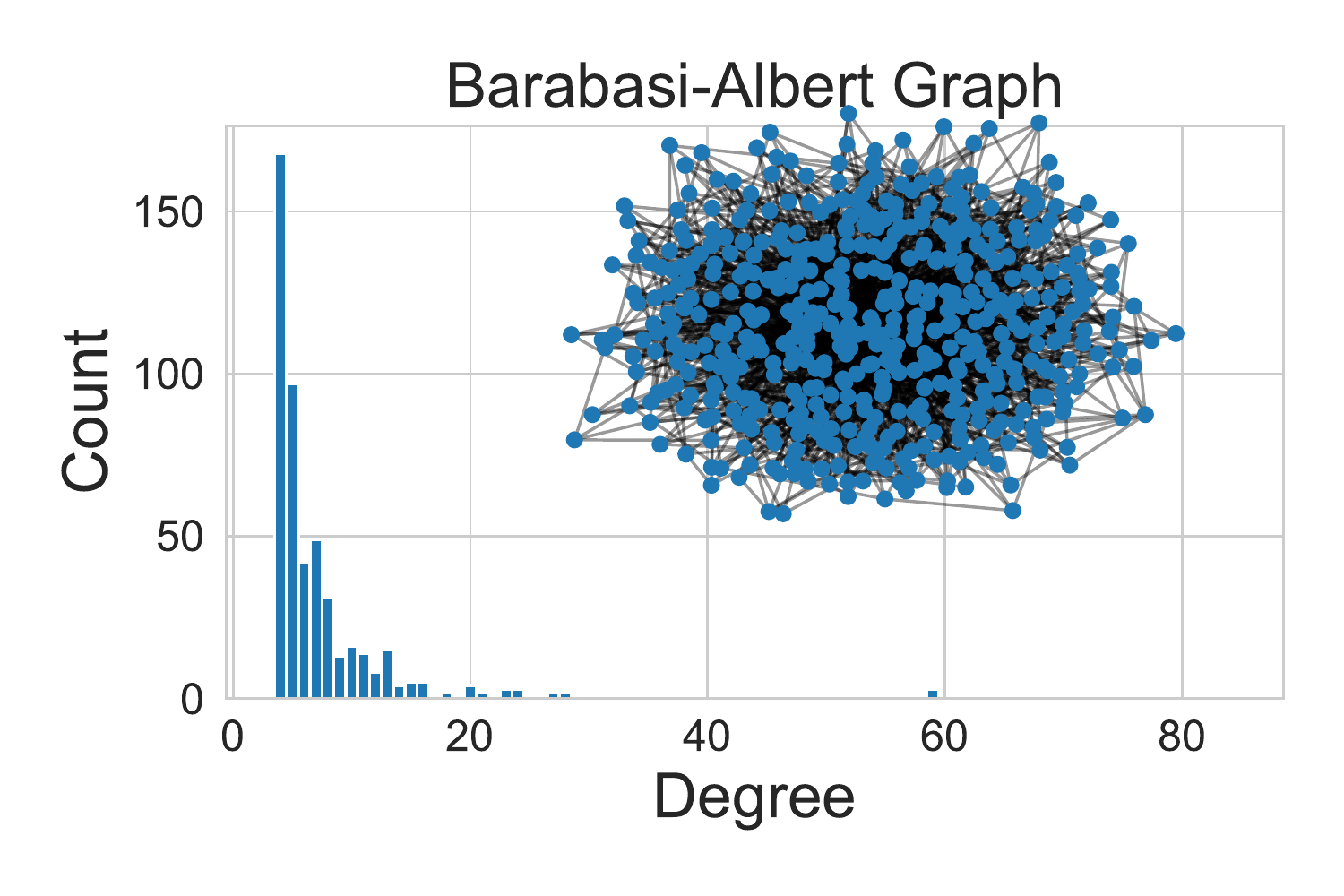}
		&
		\hspace{-5mm}
		\includegraphics[width=0.25\linewidth]{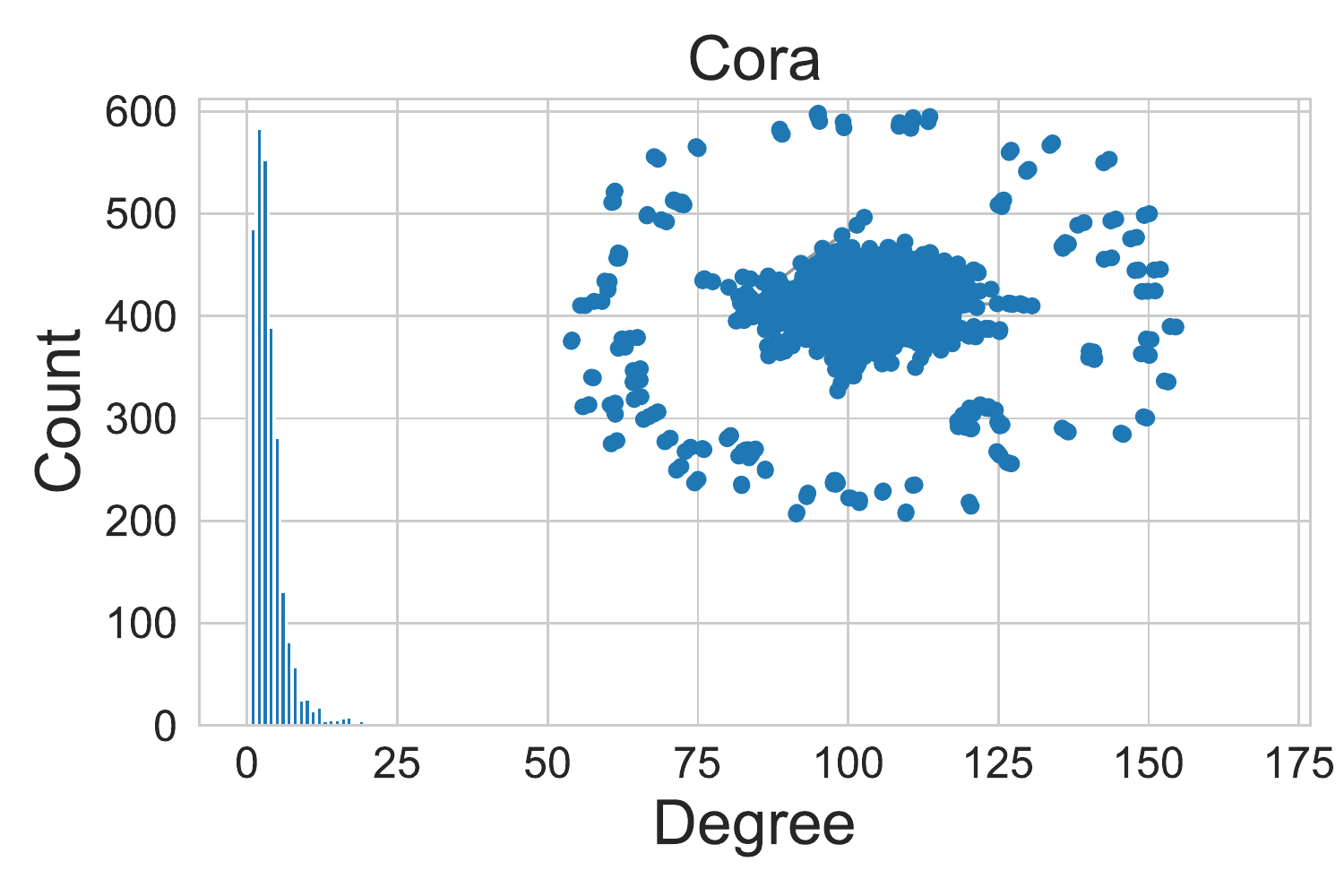}
		&
		\hspace{-5mm}
		\includegraphics[width=0.25\linewidth]{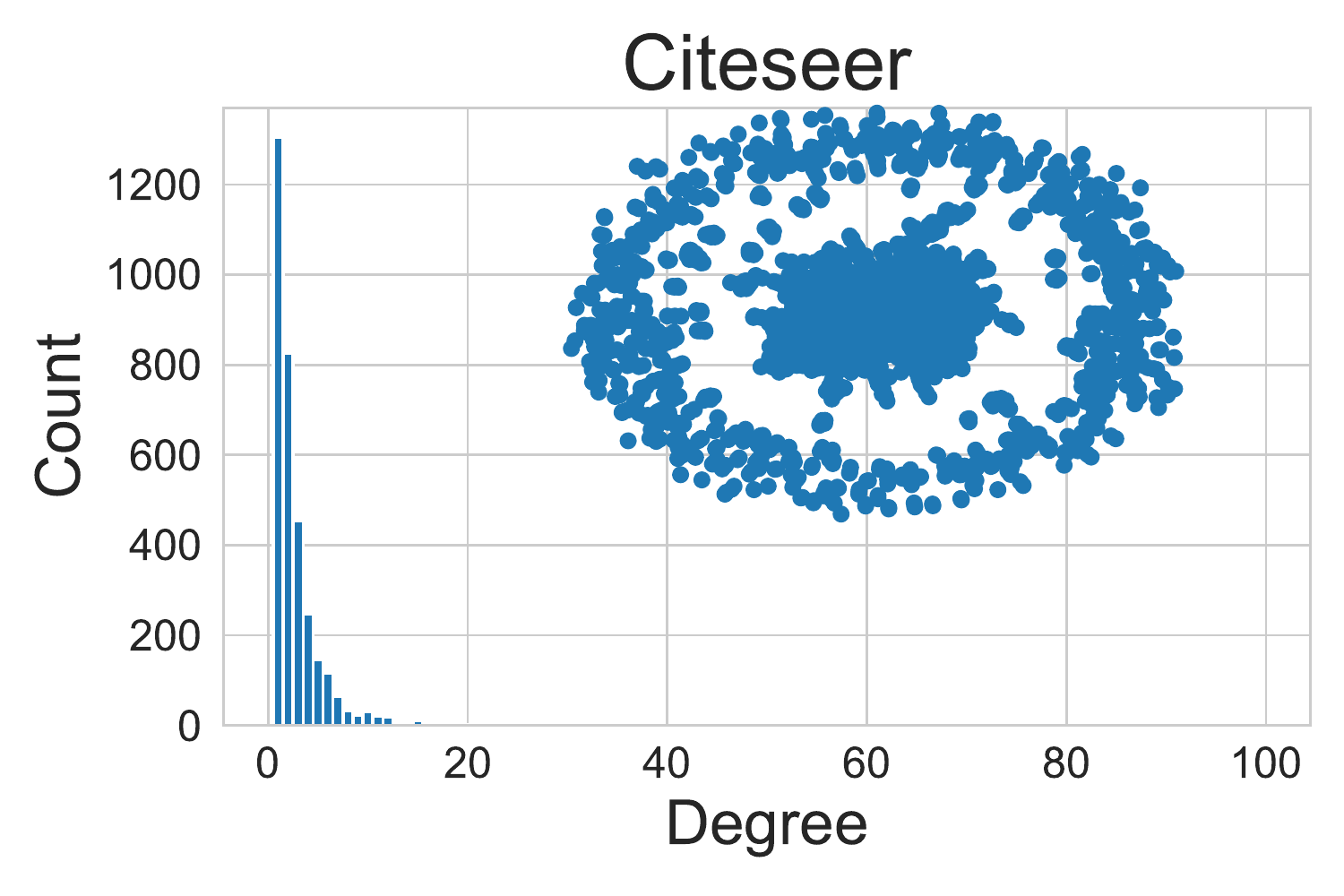}
		\\
	\end{tabular} 
	\vspace{-4mm}
	\caption{Degree statistics and visualization of different graphs.}
	\label{fig:offpolicy-graphs}
	\vspace{-4mm}
\end{figure}

\subsection{Continuous Case}
%\Roy{TODO: further check the experiment settings, make sure it is consistent with the implementation.} 
We use the Cartpole, Reacher and HalfCheetah tasks as given by OpenAI Gym.  In importance sampling, we learn a neural network policy via behavior cloning, and use its probabilities for computing importance weights $\pi_*(a|s)/\pi(a|s)$. All neural networks are feed-forward with two hidden layers of dimension $64$ and $\tanh$ activations.

\paragraph{Discrete Control Tasks} We modify the Cartpole task to be infinite horizon: We use the same dynamics as in the original task but change the reward to be $-1$ if the original task returns a termination (when the pole falls below some threshold) and $1$ otherwise. We train a policy on this task with standard Deep Q-Learning~\citep{MniKavSilGraetal13} until convergence.  

We then define the target policy $\pi_*$ as a weighted combination of this pre-trained policy (weight $0.7$) and a uniformly random policy (weight $0.3$).  The behavior policy $\pi$ for a specific $0\le \alpha\le 1$ is taken to be a weighted combination of the pre-trained policy (weight $0.55 + 0.15\alpha$) and a uniformly random policy (weight $0.45 - 0.15\alpha$). We train each stationary distribution correction estimation method using the Adam optimizer with batches of size $2048$ and learning rates chosen using a hyperparameter search from $\{0.0001, 0.0003, 0.001, 0.003\}$ and choose the best one as $0.0003$.

\paragraph{Continuous Control Tasks} For the Reacher task, we train a deterministic policy until convergence via DDPG~\citep{LilHunPriHeeetal15}.  We define the target policy $\pi$ as a Gaussian with mean given by the pre-trained policy and standard deviation given by $0.1$. The behavior policy $\pi_b$ for a specific $0\le \alpha\le 1$ is taken to be a Gaussian with mean given by the pre-trained policy and standard deviation given by $0.4 - 0.3\alpha$.  We train each stationary distribution correction estimation method using the Adam optimizer with batches of size $2048$ and learning rates chosen using a hyperparameter search from $\{0.0001, 0.0003, 0.001, 0.003\}$ and the optimal learning rate found was $0.003$).

For the HalfCheetah task, we also train a deterministic policy until convergence via DDPG~\citep{LilHunPriHeeetal15}.  We define the target policy $\pi$ as a Gaussian with mean given by the pre-trained policy and standard deviation given by $0.1$. The behavior policy $\pi_b$ for a specific $0\le \alpha\le 1$ is taken to be a Gaussian with mean given by the pre-trained policy and standard deviation given by $0.2 - 0.1\alpha$.  We train each stationary distribution correction estimation method using the Adam optimizer with batches of size $2048$ and learning rates chosen using a hyperparameter search from $\{0.0001, 0.0003, 0.001, 0.003\}$ and the optimal learning rate found was $0.003$.

%%%%%%%%%%%%%%%%%%%%%%%%%%%%%%%%%%%%%%%%%%%%%%%%%%%%%%%%%%%%%%%%%%%%%%%%%%%%
\section{Additional Experiments}\label{appendix:exp}
%%%%%%%%%%%%%%%%%%%%%%%%%%%%%%%%%%%%%%%%%%%%%%%%%%%%%%%%%%%%%%%%%%%%%%%%%%%%

\subsection{OPE for Discrete Control}
On the discrete control task, we modify the Cartpole task to be infinite horizon: the original dynamics is used but with a modified reward function: the agent will receive $-1$ if the environment returns a termination (\ie, the pole falls below some threshold) and 1 otherwise. As shown in Figure \ref{fig:offpolicy-cartpole}, our method shows competitive results with IS and Model-Based in average reward case, but our proposed method finally outperforms these two methods in terms of log MSE loss. Specifically, it is relatively difficult to fit a policy with data collected by multiple policies, which renders the poor performance of IS.
\vspace{-2mm}
\begin{figure}[ht] \centering
	\includegraphics[width=1\linewidth]{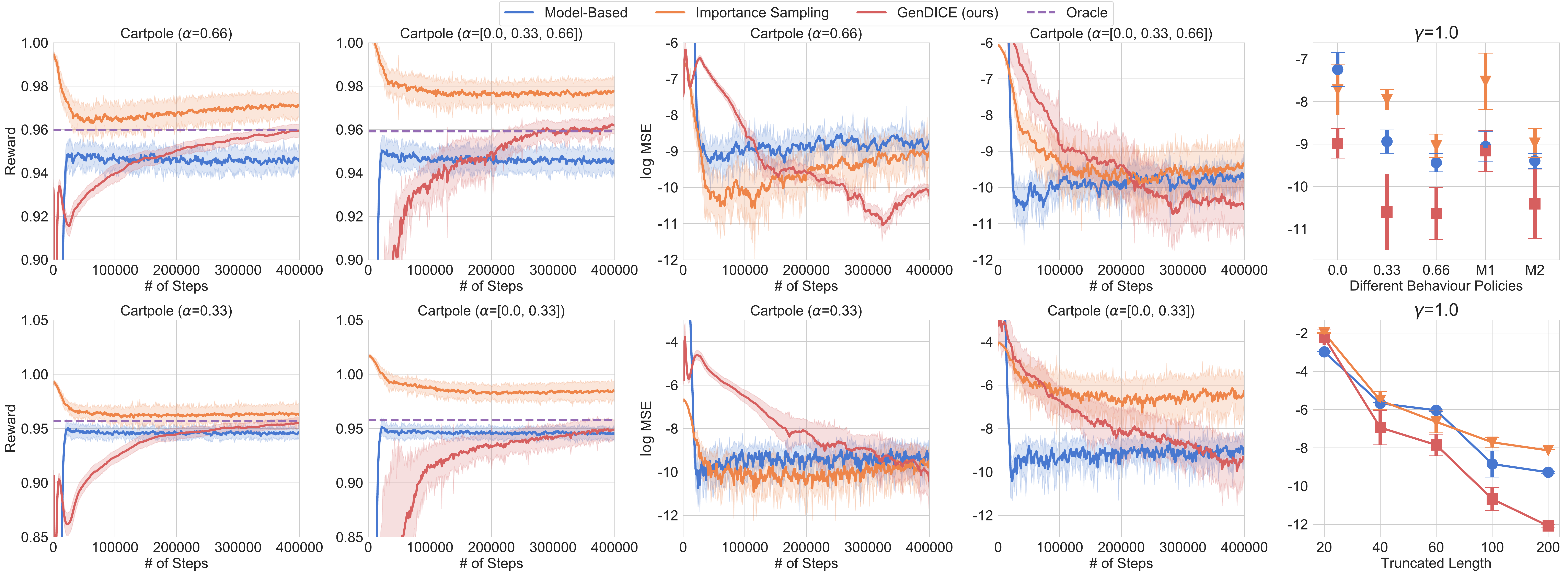}
	\vspace{-6mm}
	\caption{Results on Cartpole. Each plot in the first row shows the estimated average step reward over training and different behavior policies (higher $\alpha$ corresponds to a behavior policy closer to the target policy; the same in other figures); M1:$\alpha=[0.0, 0.33]$;  M2: $\alpha=[0.0, 0.33, 0.66]$)}
	\label{fig:offpolicy-cartpole_l}
\end{figure}

\subsection{Additional Results on Continuous Control}
In this section, we show more results on the continuous control tasks, \ie, HalfCheetah and Reacher. Figure \ref{fig:offpolicy-reacher_l} shows the $\log$ MSE towards training steps, and \estname outperforms other baselines with different behavior policies. Figure \ref{fig:offpolicy-reacher_r} better illustrates how our method beat other baselines, and can accurately estimate the reward of the target policy. Besides, Figure \ref{fig:offpolicy-half_r} shows \estname gives better reward estimation of the target policy. In these figures, the left three figures show the performance with off-policy dataset collected by single behavior policy from more difficult to easier tasks. The right two figures show the results, where off-policy dataset collected by multiple behavior policies. 

Figure \ref{fig:ablationlr_l} shows the ablation study results in terms of estimated rewards. 
The left two figures shows the effects of different learning rate. When $\alpha=0.33$, \ie, the OPE tasks are relatively easier, \estname gets relatively good results in all learning rate settings. However, when $\alpha=0.0$, \ie, the estimation becomes more difficult, only \estname in larger learning rate gets reasonable estimation. Interestingly, we can see with larger learning rates, the performance becomes better, and when learning rate is $0.001$ with $\alpha=0.0$, the variance is very high, showing some cases the estimation becomes more accurate.
The right three figures show different activation functions with different behavior policy. The square and softplus function works well; while the exponential function shows poor performance under some settings. In practice, we use the square function since its low variance and better performance in most cases.

\begin{figure}[ht] \centering
	\includegraphics[width=1\linewidth]{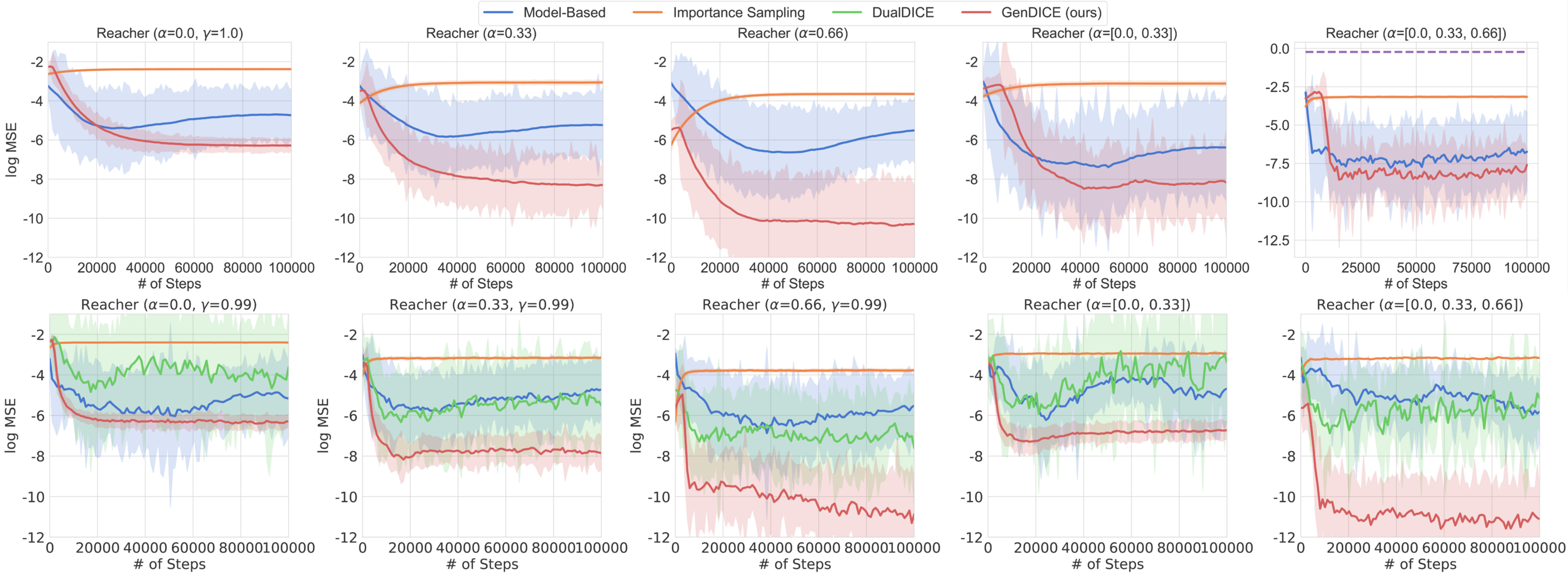}
	\vspace{-6mm}
	\caption{Results on Reacher. Each plot in the first row shows the estimated average step reward over training and different behavior policies (higher $\alpha$ corresponds to a behavior policy closer to the target policy; the same in other figures).}
	\label{fig:offpolicy-reacher_l}
\end{figure}
\begin{figure}[h!] \centering
	\includegraphics[width=1\linewidth]{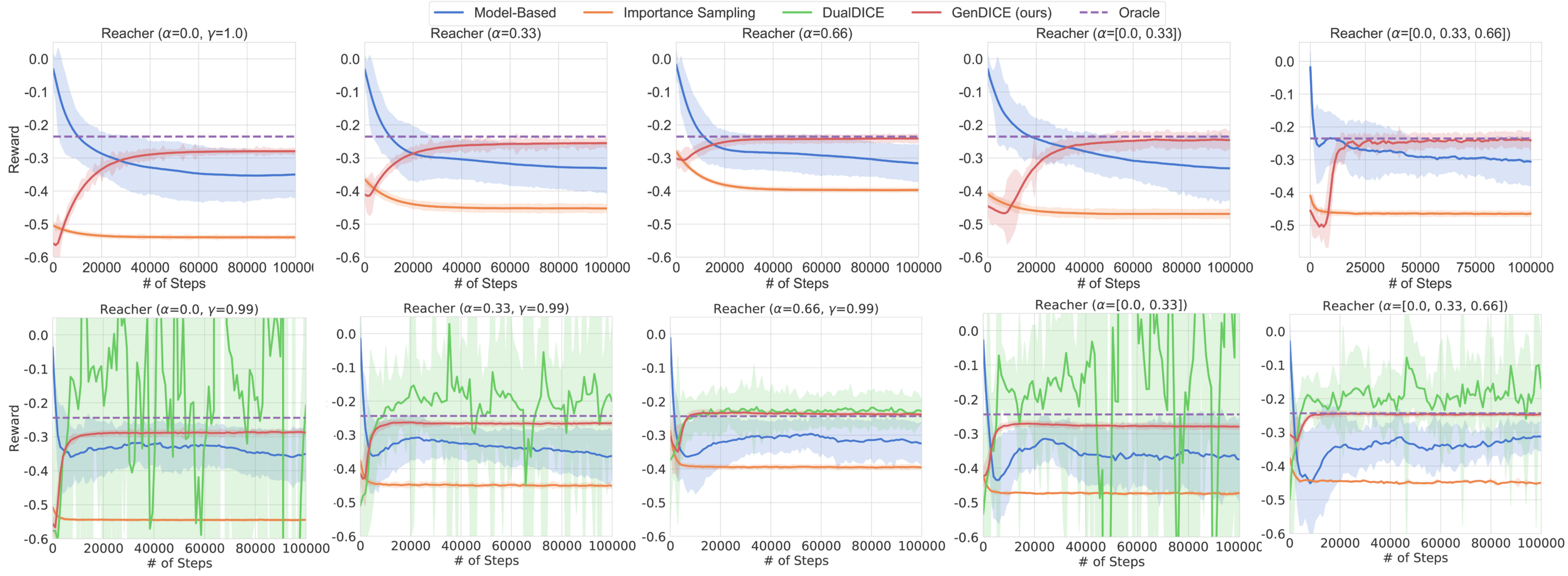}
	\vspace{-6mm}
	\caption{Results on Reacher. Each plot in the first row shows the estimated average step reward over training and different behavior policies (higher $\alpha$ corresponds to a behavior policy closer to the target policy; the same in other figures).}
	\label{fig:offpolicy-reacher_r}
\end{figure}
\begin{figure}[h!] \centering
	\includegraphics[width=1\linewidth]{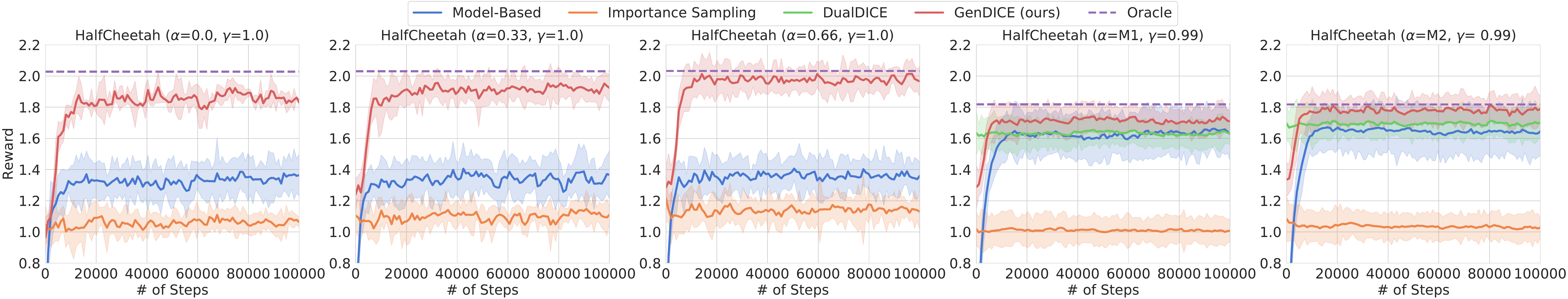}
	\vspace{-6mm}
	\caption{Results on HalfCheetah. Each plot in the first row shows the estimated average step reward over training and different behavior policies (higher $\alpha$ corresponds to a behavior policy closer to the target policy.}
	\label{fig:offpolicy-half_r}
\end{figure}
\begin{figure}[h!] \centering
	\begin{tabular}{cc}
		\hspace{-3mm}
		\includegraphics[width=0.4\linewidth]{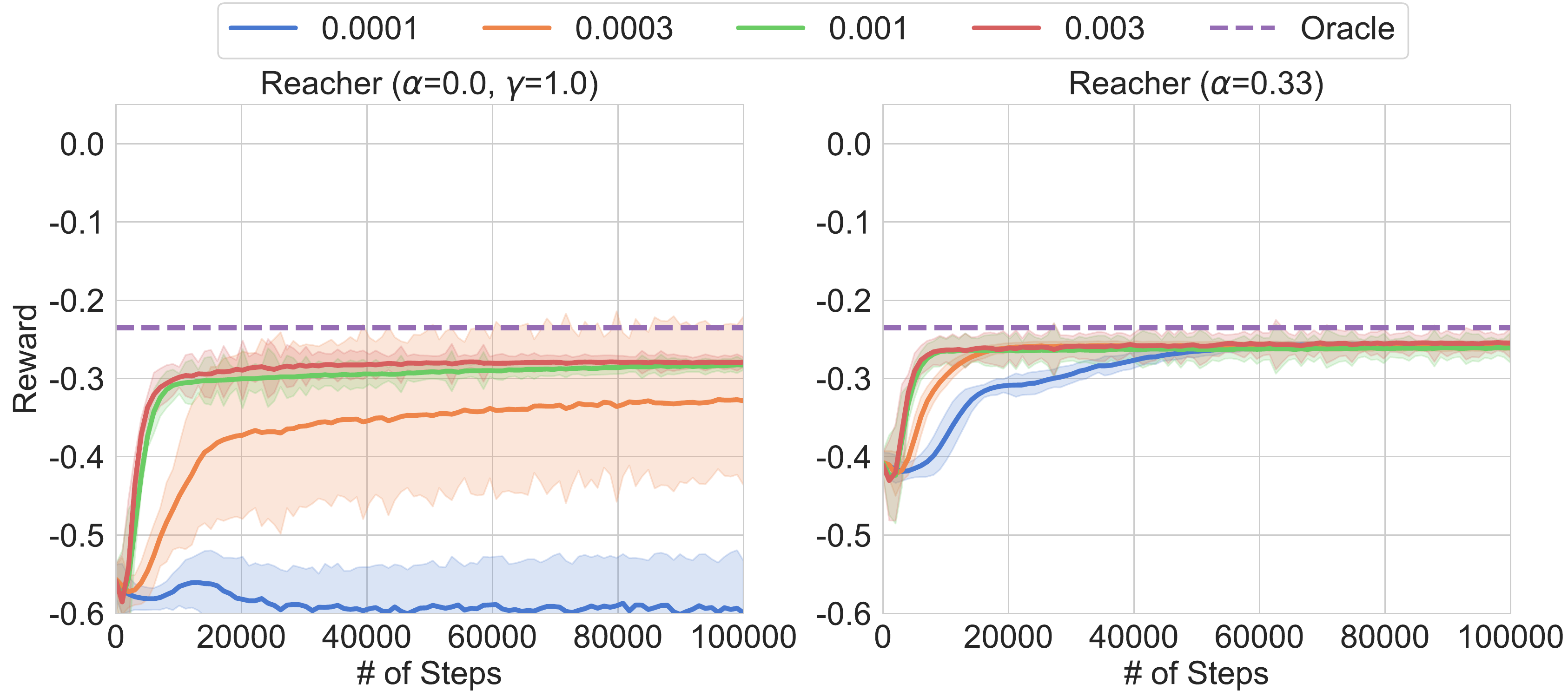}  &
		\hspace{-5mm}
  		\includegraphics[width=0.6\linewidth]{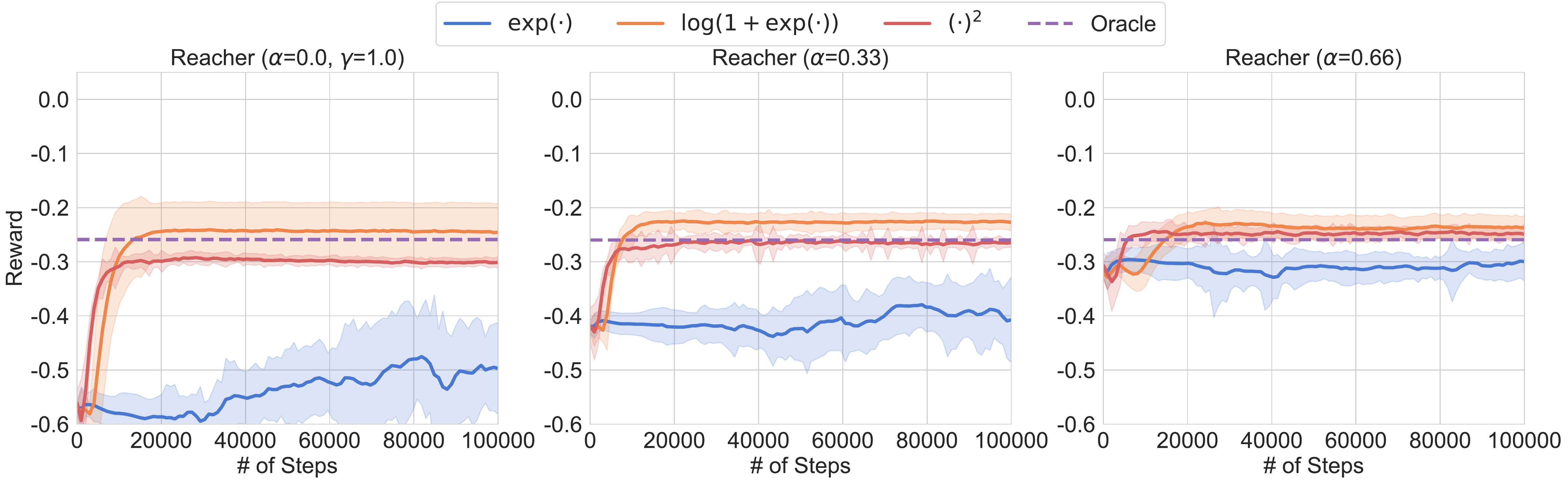}
		\\
	\end{tabular} 
	\vspace{-3mm}
	\caption{{Results of ablation study with different learning rates and activation functions. The plots show the estimated average step reward over training and different behavior policies .}}
	\label{fig:ablationlr_l}
\end{figure}

\subsection{Comparison with self-normalization trick}\label{appendix:selfnormal}

The self-normalization trick used in \cite{liu2018breaking} encodes the normalization constraint in $\tau$, while the principled optimization technique is considered in GenDICE. Further, the self-normalization trick will lead to several disadvantages theoretically, in both statistical and computational aspects :
\begin{itemize}[leftmargin=*, nosep, topsep=0pt]
	\item[{\bf i)}] It will generally not produce an unbiased solution. Although $\frac{1}{| \mathcal{D}|}\sum_{(s’, a’)\in \mathcal{D}}\tau(s’, a’)$ is an unbiased estimator for $\mathbb{E}[\tau]$, the plugin estimator $\frac{\tau(s, a)}{\frac{1}{| \mathcal{D}|}\sum_{(s’, a’)\in \mathcal{D}}\tau(s’, a’)}$ will be \emph{biased} for $\frac{\tau(s, a)}{\mathbb{E}[\tau]}$. 
	
	\item[{\bf ii)}] It will induce more computational cost. Specifically, the self-normalized ratio will be in the form of $\frac{\tau(s, a)}{\frac{1}{| \mathcal{D}|}\sum_{(s’, a’)\in \mathcal{D}}\tau(s’, a’)}$, which requires to go through all the samples in training set $\mathcal{D}$ even for just estimating one stochastic gradient, and thus, is prohibitive for large dataset. 
\end{itemize}
Empirically, self-normalization is the most natural and the first idea we tried during this project. 
We have some empirical results about this method in the OPR setting. 

\begin{wraptable}{r}{0.46\textwidth}
	\vspace{-7mm}
	\caption{Comparison between regularization and self-normalization.}
	\vspace{2mm}
	\centering
	\begin{tabular}{lc}
		% \begin{adjustbox}{scale=0.9,tabular= lcc,center}
		\toprule[1.2pt]
		& \textbf{$\log$ KL-divergence} \\
		\midrule
		self-normalization &  $-4.26 \pm 0.157  $ \\ 
		regularization &  $-4.74  \pm 0.163  $\\
		\bottomrule[1.2pt]
		\vspace{-6mm}
		% \end{adjustbox}
	\end{tabular}
	% \vspace{-4mm}
	\label{tab:oprselfaba}
\end{wraptable} 
Despite the additional computational cost, it performs worse than the proposed regularization technique used in the current version of GenDICE. 
Table \ref{tab:oprselfaba} shows a comparison between self-normalization and regularization on OPR with $\chi^2$-divergence for BA graph with 100 nodes, 10,000 offline samples, and 20 trials. We stop the algorithms in the same running-time budget. 

%\section{More Related Work}\label{appendix:more_related_work}

% Add later

% data distribution and the model distribution, 
% and can be generalized as minimizing the $f$-divergence~\citep{nowozin2016f}. 

%\vspace{-2mm}
%\paragraph{Off-policy Learning} 
%Off-policy learning~\citep{precup01off,munos2016safe,gelada2019off,liu2019off}  aims at value or policy learning from off-policy data, \ie, policy improvement given a policy, which is different from the off-policy evaluation~\citep{thomas2016data,irpan2019off}. 
%Furthermore, off-policy learning mainly focus on how to train a policy in a more stable manner with better convwergence. The off-policy evaluation can be intergrated into off-policy learning to estimate the reward of target policy for its optimization, but off-policy is not only restricted in this setting~\citep{swaminathan2017off}. Thus applying \estname in the off-policy learning setting to enhance off-polcy learning is an interesting future direction.

\end{appendix}

\end{document}